\documentclass[11pt]{article}

\usepackage[preprint]{acl}
\usepackage{times}
\usepackage{latexsym}
\usepackage[T1]{fontenc}
\usepackage[utf8]{inputenc}
\usepackage{microtype}
\usepackage{inconsolata}
\usepackage{graphicx}
\usepackage{booktabs} 
\usepackage{multirow}
\usepackage{placeins}

\usepackage{amsmath}
\usepackage{amssymb}
\usepackage{mathtools}
\usepackage{amsthm}

\usepackage[capitalize,noabbrev]{cleveref}

\theoremstyle{plain}
\newtheorem{theorem}{Theorem}
\newtheorem{proposition}[theorem]{Proposition}
\newtheorem{lemma}[theorem]{Lemma}

\theoremstyle{definition}
\newtheorem{definition}{Definition}
\newtheorem{assumption}{Assumption}
\theoremstyle{remark}
\newtheorem{remark}{Remark}
\usepackage{algorithm}
\usepackage{algorithmic} 
\newcommand{\Algphase}[1]{%
  \item[]\hrulefill%
  \item[]\textbf{#1}%
  \item[]%
}

\newcommand{\best}[1]{\textbf{#1}}
\usepackage{siunitx}
\usepackage{xcolor}
\title{ZOIF: Zeroth-Order Influence Function Estimation for LLM Data Selection Using Only Forward Passes}
\author{Lisong He\thanks{\scriptsize\texttt{helisong@eurasia.edu} (Lisong He)} \quad
Yunxuan Jiang\thanks{\scriptsize\texttt{fengjianliu@stu.xjtu.edu.cn} (Yunxuan Jiang)} \quad
Haishan Ye\thanks{\scriptsize\texttt{yehaishan@xjtu.edu.cn} (Haishan Ye)} \quad
Xiangyu Chang\thanks{\scriptsize\texttt{xiangyuchang@gmail.com} (Xiangyu Chang; corresponding author)}\\
School of Management, Xi'an Jiaotong University, Xi'an, China}

\begin{document}

\maketitle

\begin{abstract}
Influence functions enable principled data selection for large language models (e.g., identifying influential training samples for model editing, debugging, or efficient fine-tuning). 
However, scaling IFs to LLMs is hindered by the intractability of billion-scale Hessian matrices and the prohibitive memory overhead of backpropagation (BP). While many classical IF approximations reduce Hessian complexity, they still rely on BP for gradient computation, requiring massive storage for intermediate activations that impedes practical data selection at scale. 
To address this, we propose \texttt{ZOIF}, a zeroth-order framework for influence estimation that operates entirely via forward passes, enabling memory-efficient data selection under resource constraints. 
\texttt{ZOIF} decomposes influence estimation into three BP-free stages: (1) approximating the average proxy gradient via zeroth-order estimation; (2) reformulating and solving the inverse Hessian-vector product as an auxiliary zeroth-order optimization to bypass explicit inversion; and (3) deriving per-sample influence scores via a forward finite difference along the previously computed parameter update direction. 
Experiments show that \texttt{ZOIF} selects data subsets with downstream utility competitive with---and in some cases exceeding---the evaluated baselines, yet at substantially lower peak memory cost and with comparable overall computation time.
\end{abstract}

\section{Introduction}

As large language models (LLMs) continue to demonstrate remarkable performance across a wide range of tasks, the need for efficient, scalable, and principled data selection—such as for fine-tuning, debugging, and model editing—has become increasingly urgent.
Influence functions (IFs) originated in robust statistical diagnostics~\citep{Hampel1974} and were later adapted to modern machine learning~\citep{koh2017understanding}. They provide a principled framework for quantifying how individual training samples influence a model’s predictions on test points, thereby offering a theoretically grounded foundation for data attribution and selection~\citep{grosse2023studying,choe2024data,pan2024gdig,xia2024less,liu2021influence}.

Formally, IFs quantify the effect of an infinitesimal upweighting of a training sample on a target loss by leveraging a local second-order approximation of the empirical risk.
Given a pretrained model with parameters $\boldsymbol{\theta}_0$, we consider the following local Hessian-preconditioned influence score, which coincides with the classical influence function when $\boldsymbol{\theta}_0$ is a stationary solution of the reference objective:
\begin{equation}\label{eq:IF}
\mathcal{I}\left(z_i, z_j^{\prime}\right) = -\nabla_{\boldsymbol{\theta}} \ell(\boldsymbol{\theta}_0, z_i)^\top \boldsymbol{H}_{\boldsymbol{\theta}_0}^{-1} \nabla_{\boldsymbol{\theta}} \ell(\boldsymbol{\theta}_0,z'_j),
\end{equation}
where $\ell$ is the loss function and  $\boldsymbol{H}_{\boldsymbol{\theta}_0}$ denotes the empirical Hessian of the training objective evaluated at $\boldsymbol{\theta}_0$.
This formulation reveals that influence estimation fundamentally requires solving an inverse Hessian--vector product (IHVP), $\boldsymbol{H}_{\boldsymbol{\theta}_0}^{-1} \nabla_{\boldsymbol{\theta}} \ell(\boldsymbol{\theta}_0,z'_j)$, which constitutes the primary computational and memory bottleneck when scaling influence functions to modern LLMs.
To mitigate this challenge, recent studies have proposed empirical-Fisher or structured-curvature approximations for influence computation~\citep{Kwonetal2023,grosse2023studying}, low-dimensional parameter or curvature subspaces~\citep{schioppa2022scaling}, and random or structured projections of gradient features~\citep{park2023trak,choe2024data}. Related second-order approximations have also been explored in model unlearning~\citep{zhang2024soul}.

However, most scalable approximations of classical IFs still rely on first-order (FO) gradients, namely
$\nabla_{\boldsymbol{\theta}} \ell\!\left(\boldsymbol{\theta}_0, z_i\right)$
and
$\nabla_{\boldsymbol{\theta}} \ell\!\left(\boldsymbol{\theta}_0, z_j^{\prime}\right)$, which necessitate backpropagation (BP) through the full computational graph.
This dependence incurs prohibitive memory overhead; for instance, computing gradients via BP for OPT-13B requires more than $12\times$ the memory footprint of a single forward pass~\citep{zhang2024revisiting}.
Under limited GPU memory budgets (common in real-world deployment and edge scenarios), such memory demands severely hinder the practical applicability of influence functions for large-scale data selection.

These limitations motivate the central research question of this work: \emph{How can we estimate IFs for LLMs in a memory-efficient manner, without relying on BP or explicit second-order information, while retaining their effectiveness for data selection?}

To address this question, we propose {\texttt{ZOIF}} (\textbf{Z}eroth-\textbf{O}rder \textbf{I}nfluence \textbf{F}unction Estimation), a novel framework for influence estimation that operates exclusively through forward passes of the model.
The key idea underlying \texttt{ZOIF} is to systematically replace all gradient and Hessian-dependent components in classical influence functions with ZO estimators, thereby eliminating the need for BP and the storage of intermediate activations.

At a high level, \texttt{ZOIF} decomposes influence estimation into three memory-efficient stages.
First, it approximates the average loss gradient over a small validation subset with ZO evaluations, reducing activation memory while trading it for additional forward passes.
Second, \texttt{ZOIF} casts the IHVP as the solution of an auxiliary optimization problem and solves it with forward-only ZO updates.
Finally, for each candidate training sample, \texttt{ZOIF} estimates its influence score via a forward finite difference along the previously computed parameter update direction.
As designed above, the entire \texttt{ZOIF} pipeline enables scalable, practical influence-based data selection for LLMs under strict memory constraints.
The framework of \texttt{ZOIF} is illustrated in Figure~\ref{fig:Framework}.

\begin{figure*}[t]
    \centering
    \includegraphics[width=0.92\textwidth]{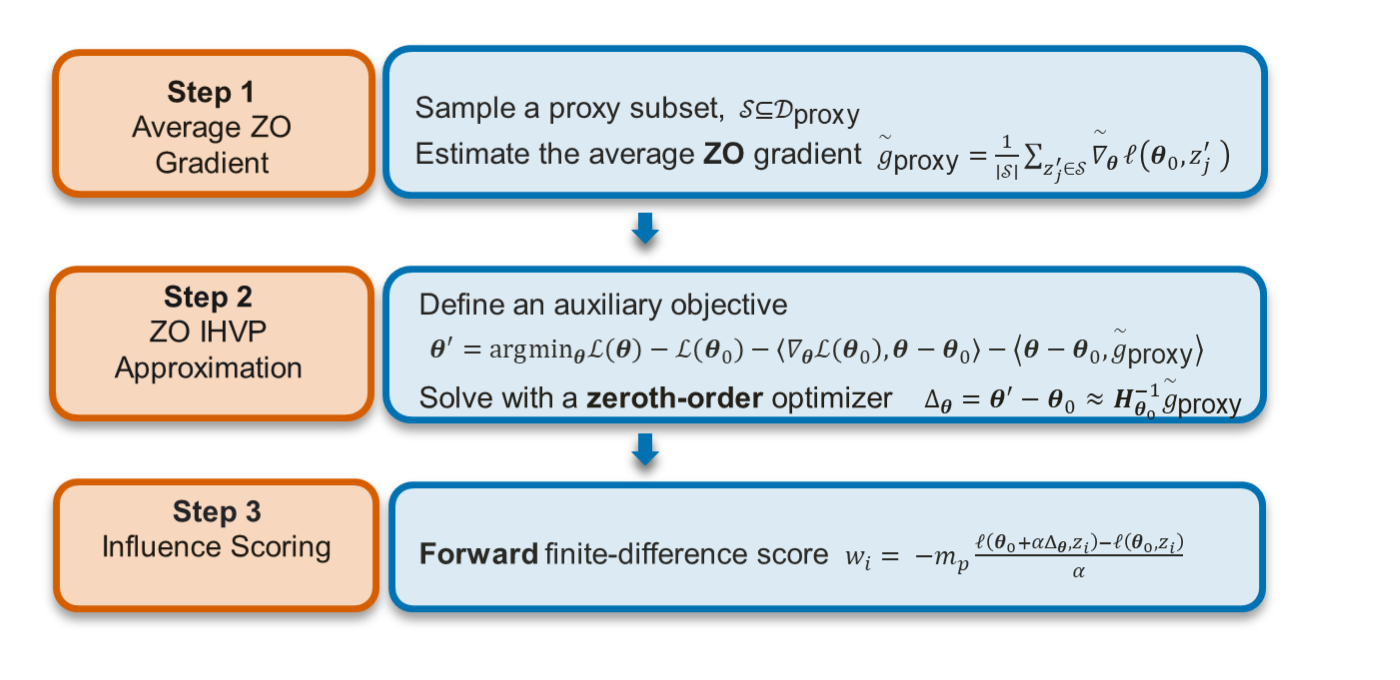}
    \caption{\texttt{ZOIF} framework and information flow across its three forward-only steps.}
    \label{fig:Framework}
\end{figure*}
The main contributions of this work can be summarized as follows:
\begin{itemize}
     \item  We introduce \texttt{ZOIF}, a novel ZO framework for IF estimation that operates exclusively via forward passes. 
    By systematically integrating ZO techniques, \texttt{ZOIF} bypasses BP and reduces the activation-memory burden of influence estimation for large-scale models.
    
    \item  We propose a principled approach to compute the IHVP by reformulating it as an auxiliary ZO optimization problem.
    This effectively eliminates the need for computationally expensive high-dimensional matrix inversions. Furthermore, our framework is designed to be highly compatible with various existing influence function approximation techniques, offering significant flexibility for practical deployment across different tasks and resource constraints.
    
    \item Experiments show that \texttt{ZOIF} selects data subsets with downstream utility competitive with---and in some cases exceeding---the evaluated baselines, while substantially lowering peak memory usage and maintaining comparable overall computation time.
\end{itemize}

\section{Related Work}\label{relatedwork}

During large-scale pretraining and fine-tuning, efficiently identifying high-value samples from massive data pools has become a critical challenge for improving model training efficiency and generalization capability. Existing data selection methods can be broadly categorized into three types: (1) heuristic approaches based on statistical or model-based metrics~\citep{yin2024compute, antonello2021selecting, marion2023when, li2024from}; (2) similarity-based methods leveraging semantic representations~\citep{ni2022large, lee2024nv, ivison2025large}; and (3) influence-based methods, including gradient-based and forward-only estimators~\citep{pruthi2020estimating,wu2024greats,nikdan2025efficient,ko2024mirrored,kokhlikyan2025z0inf,deng2026forvalue}. This paper focuses primarily on the third category and integrates ZO techniques to alleviate the memory bottleneck commonly encountered when applying IFs in large-scale models.

\subsection{Approximate Computation of IFs}

Scaling IFs to LLMs is primarily limited by Hessian computation, motivating several approximation strategies.
One prominent line of research simplifies or discards second-order information to reduce complexity. For instance, TracIn~\citep{pruthi2020estimating} approximates influence values using inner products of first-order gradients and omits the Hessian term. Influence Distillation retains second-order information when deriving sample weights and introduces a landmark-based approximation for scalable data selection~\citep{nikdan2025efficient}. DataInf employs a damped, layer-wise empirical-Fisher approximation with an efficient closed-form inverse approximation~\citep{Kwonetal2023}, whereas SOUL uses diagonal second-order information in the related setting of model unlearning~\citep{zhang2024soul}.
Another strategy compresses the parameter, curvature, or gradient representation used for influence analysis. Existing methods use low-dimensional parameter or curvature subspaces, structured curvature factorizations, random projections of gradient features, or structured gradient projections~\citep{schioppa2022scaling,grosse2023studying,park2023trak,choe2024data}.
Another line of research focuses on efficiently approximating the IHVP. Classical iterative methods such as conjugate gradient (CG) and LISSA~\citep{agarwal2017second} use repeated Hessian--vector product (HVP) evaluations without explicitly constructing the Hessian. DAUNCE expresses the same local second-order structure as an auxiliary optimization problem~\citep{pan2025daunce}, which motivates the formulation used in Step~2.

\subsection{Zeroth-Order Influence Estimation}
ZO estimation approximates gradients from forward finite differences along random parameter-space perturbations, avoiding the BP graph and reducing activation memory for resource-constrained LLM fine-tuning~\citep{Flaxman2004Online,Nesterov2017Random,malladi2023mezo}.

Recent forward-only estimators differ in signal and target. Mirrored Influence retains gradients for a small test/proxy set and scores training samples from forward losses~\citep{ko2024mirrored}; Z0-Inf uses checkpoint-trajectory loss changes~\citep{kokhlikyan2025z0inf}; and For-Value forms a closed-form valuation from final-layer representations and prediction errors~\citep{deng2026forvalue}. \texttt{ZOIF} instead decomposes the classical local Hessian-preconditioned IF into three components that admit ZO estimation: it estimates the proxy gradient with ZO perturbations, obtains the IHVP direction through ZO-based auxiliary optimization, and evaluates each candidate by a forward finite difference along the resulting shared direction. Thus, despite forward-only candidate evaluation, these methods differ in estimand, required information, and computation.

\section{Background and Problem Setup}\label{problem}

\subsection{Data Selection for LLMs}
\label{sec:data_selection}

Let an LLM be characterized by a parameter vector \(\boldsymbol{\theta} \in \mathbb{R}^d\), where \(d\) denotes the total number of trainable parameters. 
Given a data point \(z = (x, y) \in \mathcal{X} \times \mathcal{Y}\), let \(\ell(\boldsymbol{\theta}, z)\) denote the loss incurred by the model with parameters \(\boldsymbol{\theta}\) on \(z\).

Given a candidate data pool $\mathcal{D}_{\text{pool}} = \{z_i = (x_i, y_i)\}_{i=1}^n \subset \mathcal{X} \times \mathcal{Y}$ and an independent validation dataset $\mathcal{D}_{\text{val}} = \{z_j' = (x_j', y_j')\}_{j=1}^m \subset \mathcal{X} \times \mathcal{Y}$, we use $\mathcal{D}_{\text{val}}$ to evaluate the model's generalization performance on the target task.
To avoid overfitting to $\mathcal{D}_{\text{val}}$ during the data selection process, it is partitioned into two mutually exclusive subsets:
a proxy set $\mathcal{D}_{\text{proxy}} \subset \mathcal{D}_{\text{val}}$, used to guide data selection, with $|\mathcal{D}_{\text{proxy}}| = m_p \ll m$;
and a hold-out test set $\mathcal{D}_{\text{test}} = \mathcal{D}_{\text{val}} \setminus \mathcal{D}_{\text{proxy}}$, reserved solely for final unbiased evaluation and kept unseen during the selection phase.
The goal of data selection is to choose a subset $\mathcal{D}_{\text{select}} \subseteq \mathcal{D}_\text{{pool }}$ of size $n_s$ such that the model fine-tuned on this subset achieves optimal performance on $\mathcal{D}_{\text{proxy}}$.
\begin{definition}
Let $\boldsymbol{\theta}_0$ denote the pretrained initial parameters, and let $\mathcal{A}\left(\boldsymbol{\theta}_0, \mathcal{D}\right)$ represent the final parameters obtained by applying a fine-tuning algorithm starting from $\boldsymbol{\theta}_0$ using subset $\mathcal{D}$ as training data. The proxy loss is defined as
\begin{equation}
\mathcal{L}_{\text{proxy}}(\boldsymbol{\theta})=\frac{1}{m_p} \sum_{z_j^{\prime} \in \mathcal{D}_{\text{proxy}}} \ell\left(\boldsymbol{\theta}, z_j^{\prime}\right).
\end{equation}

The data selection problem can then be rigorously formulated as the following optimization problem:
\begin{equation}
\mathcal{D}_{\text{select}}\in\underset{\mathcal{D} \subseteq \mathcal{D}_{\text{pool}},|\mathcal{D}|=n_s}{\arg \min }\mathcal{L}_{\text{proxy}} \left(\mathcal{A}\left(\boldsymbol{\theta}_0, \mathcal{D}\right)\right).
\end{equation}
\end{definition}

\subsection{Influence Functions}

IFs~\citep{koh2017understanding} have emerged as a prominent tool for data selection in recent years, owing to their solid theoretical foundation and strong interpretability.
IFs quantify how an infinitesimal upweighting of a training sample \( z_i \) (i.e., increasing its weight in the empirical risk) affects the model parameters and the loss on a test point.

The influence score is commonly computed as:

\begin{equation}
\mathcal{I}\left(z_i, z_j^{\prime}\right) = -\nabla_{\boldsymbol{\theta}} \ell(\boldsymbol{\theta}_0, z_i)^\top \boldsymbol{H}_{\boldsymbol{\theta}_0}^{-1} \nabla_{\boldsymbol{\theta}} \ell(\boldsymbol{\theta}_0,z'_j),
\end{equation}

where \( z_i \) is an arbitrary sample from the data pool, \( z_j^{\prime} \) is an arbitrary sample from \( \mathcal{D}_{\text{proxy}} \), \( \nabla_{\boldsymbol{\theta}} \ell(\boldsymbol{\theta}_0, z_i) \in \mathbb{R}^d \) is the gradient of the loss with respect to the model parameters, \( \boldsymbol{H}_{\boldsymbol{\theta}_0} = \frac{1}{n} \sum_{i=1}^n \nabla_{\boldsymbol{\theta}}^2 \ell (\boldsymbol{\theta}_0, z_i) \in \mathbb{R}^{d \times d} \) is the empirical Hessian matrix.

The classical upweighting interpretation additionally assumes that \(\boldsymbol{\theta}_0\) is a stationary solution of the reference training objective. In practice, one may optionally warm up the model on the reference data so that \(\|\nabla\mathcal{L}(\boldsymbol{\theta}_0)\|\) is small. When this warm-up is omitted, as in our experiments to avoid additional training cost, we use the same Hessian-preconditioned gradient similarity as a \emph{local influence proxy} at the given pretrained checkpoint. It recovers the damped classical IF interpretation when the stationarity residual vanishes.

Specifically, if \( \mathcal{I}(z_i, z'_j) < 0 \), then \( z_i \) is a helpful sample whose presence improves model performance on \( z'_j \). 
By sorting the influence values from smallest to largest, we can prioritize and select the most beneficial training points at the top of the list.

\subsection{Zeroth-Order Estimation}

Although Hessian approximations avoid explicit construction of a \(d \times d\) matrix, conventional IF methods still rely on BP, which stores gradients and intermediate activations and therefore incurs substantial GPU memory overhead for LLMs.
We accordingly adopt ZO gradient estimation, which uses only forward evaluations, to reduce this memory cost.

The classical ZO gradient estimator is based on Gaussian smoothing~\citep{Nesterov2017Random}. For a target function \( \ell(\cdot, z): \mathbb{R}^d \to \mathbb{R} \), the forward-difference-based Gaussian-smoothed gradient estimator is defined as:
\begin{equation}\label{eq:zo_gradient}
\begin{aligned}
\widetilde{\nabla}_{\boldsymbol{\theta}}
\ell(\boldsymbol{\theta}_0,z)
&=\frac{1}{K}\sum_{k=1}^{K}
\frac{\ell(\boldsymbol{\theta}_0+2\alpha\boldsymbol{u}_k,z)}{2\alpha}
\boldsymbol{u}_k \\
&\quad-\frac{1}{K}\sum_{k=1}^{K}
\frac{\ell(\boldsymbol{\theta}_0,z)}{2\alpha}
\boldsymbol{u}_k .
\end{aligned}
\end{equation}
where \( \boldsymbol{u}_k \sim \mathcal{N}(0, I_d) \), \( \alpha > 0 \) is the smoothing radius, and \( K \) is the number of random directions sampled. Its Monte Carlo expectation equals the gradient of the corresponding Gaussian-smoothed objective; it is not generally unbiased for the unsmoothed gradient at finite \(\alpha\). The estimator requires only \( K+1 \) forward evaluations, completely avoiding BP and thereby reducing memory overhead.

\section{Methodology}\label{method}

\subsection{Decomposition of IF Computation}
In many practical applications, we are not concerned with the isolated influence of a single training sample on a single proxy sample, but rather with the aggregate influence of each training sample on the entire proxy set. This quantity equals the corresponding row sum of the influence matrix $\mathbf{\mathcal{I}}=(\mathcal{I}(z_i, z'_j)) \in \mathbb{R}^{n \times m_p}$.
Thus, the core question is \textit{How can we efficiently compute the row sums of the influence matrix without explicitly constructing the full matrix $\mathcal{I}(z_i, z'_j) $ ?}

We define the total influence of the $i$-th training sample $z_i$ on the entire proxy set as 
\begin{small}
\begin{align*} 
\mathcal{I}_{\text {total }}\left(z_i\right)&:=-\sum_{j=1}^{m_p} \nabla_{\boldsymbol{\theta}} \ell\left(\boldsymbol{\theta}_0, z_i\right)^T \boldsymbol{H}_{\boldsymbol{\theta}_0}^{-1} \nabla_{\boldsymbol{\theta}} \ell\left(\boldsymbol{\theta}_0, z_j^{\prime}\right)\\
& =-m_p \nabla_{\boldsymbol{\theta}} \ell\left(\boldsymbol{\theta}_0, z_i\right)^T \boldsymbol{H}_{\boldsymbol{\theta}_0}^{-1}\\
&\quad \cdot\left(\frac{1}{m_p} \sum_{j=1}^{m_p} \nabla_{\boldsymbol{\theta}} \ell\left(\boldsymbol{\theta}_0, z_j^{\prime}\right)\right)\\
& =-m_p \nabla_{\boldsymbol{\theta}} \ell\left(\boldsymbol{\theta}_0, z_i\right)^T \boldsymbol{H}_{\boldsymbol{\theta}_0}^{-1} \boldsymbol{g}_{\text{proxy}},
\end{align*}
\end{small}
where $\boldsymbol{g}_{\text{proxy}}=\frac{1}{m_p} \sum_{j=1}^{m_p} \nabla_{\boldsymbol{\theta}} \ell\left(\boldsymbol{\theta}_0, z_j^{\prime}\right)$.

The above formulation provides the following three-step process for the decomposition of total influence computations.
\begin{itemize}
    \item First, compute the average of gradients $\boldsymbol{g}_{\text{proxy}}$ across all proxy samples once.
    \item Second, compute a single IHVP
$\boldsymbol{v}=\boldsymbol{H}_{\boldsymbol{\theta}_0}^{-1} \boldsymbol{g}_{\text{proxy}}$.
    \item Third, for each training sample $z_i$, obtain its total influence via a simple dot product $-m_p \nabla_{\boldsymbol{\theta}} \ell\left(\boldsymbol{\theta}_0, z_i\right)^T \boldsymbol{v}$.
\end{itemize}

\subsection{\texttt{ZOIF}: Zeroth-Order Influence Estimation}

Building on the above decomposition, we propose a fully ZO, three-stage computational framework to efficiently estimate IFs for LLMs under memory-constrained conditions. 
The entire pipeline operates without explicit access to gradients or Hessian matrices, relying solely on model forward passes.

\subsubsection{Step 1: Average ZO Gradient Estimation }

Directly computing the full gradient averaged over all \( m_p \) proxy samples is computationally prohibitive under memory-constrained conditions. 
To address this, we proceed as follows:
First, we uniformly sample a subset \( \mathcal{S} \subseteq \mathcal{D}_{\text{proxy}} \) of size \( |\mathcal{S}| \) from the proxy dataset.  
Then, for each \( z_j' \in \mathcal{S} \), we approximate the gradient using \texttt{ZOGE} (Zeroth-Order Gradient Estimation; see the supplementary material), following~\citet{malladi2023mezo}\footnote{We adopt the gradient estimator from ZO optimization but replace central differences with forward differences, reducing function evaluations from \(2K\) to \(K+1\) for \(K\) random directions.}, and aggregate the results into an averaged gradient estimate:
\begin{equation}
\tilde{g}_{\text{proxy}} = \frac{1}{|\mathcal{S}|} \sum_{z_j' \in \mathcal{S}} \tilde{\nabla}_{\boldsymbol{\theta}} \ell\left(\boldsymbol{\theta}_0, z_j'\right)
\end{equation}
where $\tilde{\nabla}_{\boldsymbol{\theta}} \ell\left(\boldsymbol{\theta}_0, z_j'\right)$ is denoted in Eq.\eqref{eq:zo_gradient}.

The trade-off between computational cost and estimation variance can be controlled by adjusting $|\mathcal{S}|$.
In particular, when the proxy dataset size \( m_p \) is small, we can set \( |\mathcal{S}| = m_p \) to remove proxy-subsampling error. The remaining discrepancy from the unsmoothed average gradient consists of Monte Carlo sketching error and finite-difference/smoothing bias.

\subsubsection{Step 2: Solving the IHVP via ZO Optimization}

Following the auxiliary optimization formulation introduced by DAUNCE~\citep{pan2025daunce}, we recast the computation of the IHVP as an optimization problem that can be solved entirely with ZO updates.

Let \(\mathcal{L}(\boldsymbol{\theta})=\frac{1}{n}\sum_{z\in\mathcal{D}_{\text{pool}}}\ell(\boldsymbol{\theta},z)\), and estimate the reference gradient \(\tilde{g}_0\) using only forward evaluations. Starting from \(\boldsymbol{\theta}_0\), we seek a local stationary solution \(\boldsymbol{\theta}'\) of the objective
\begin{small}
\begin{equation}
\begin{split}
\mathcal{J}(\boldsymbol{\theta}) = \big( & \mathcal{L}(\boldsymbol{\theta}) - \mathcal{L}\left(\boldsymbol{\theta}_0\right)  - \left\langle\tilde{g}_0, \boldsymbol{\theta}-\boldsymbol{\theta}_0\right\rangle \\
 & - \left\langle\boldsymbol{\theta}-\boldsymbol{\theta}_0, \tilde{g}_{\text{proxy}}\right\rangle \big).\label{proxyopt}
\end{split}
\end{equation}
\end{small}
This optimization problem is derived from the second-order Taylor expansion of the loss function, where

\begin{small}
\begin{equation}
\begin{split}
\mathcal{L}(\boldsymbol{\theta})-&\mathcal{L}\left(\boldsymbol{\theta}_0\right)- \left\langle\nabla_{\boldsymbol{\theta}} \mathcal{L}\left(\boldsymbol{\theta}_0\right), \boldsymbol{\theta}-\boldsymbol{\theta}_0\right\rangle \\
& \approx \frac{1}{2}\left(\boldsymbol{\theta}-\boldsymbol{\theta}_0\right)^{\top} \boldsymbol{H}_{\boldsymbol{\theta}_0}\left(\boldsymbol{\theta}-\boldsymbol{\theta}_0\right).
\end{split}
\end{equation}
\end{small}

After local optimization, the resulting parameter change $\Delta_{\boldsymbol{\theta}}=\boldsymbol{\theta}^{\prime}-\boldsymbol{\theta}_0$ approximates $\boldsymbol{H}_{\boldsymbol{\theta}_0}^{-1}\tilde{g}_{\text{proxy}}$ when the local curvature is positive definite and the stationarity residual is small. Both the reference gradient and the auxiliary-objective updates are estimated by \texttt{ZOGE}, and we project $\Delta_{\boldsymbol{\theta}}$ onto a small Euclidean ball for numerical stability. When additional stabilization is needed, adding $\lambda\|\boldsymbol{\theta}-\boldsymbol{\theta}_0\|_2^2/2$ yields the damped direction $(\boldsymbol{H}_{\boldsymbol{\theta}_0}+\lambda\boldsymbol{I})^{-1}\tilde{g}_{\text{proxy}}$.

Both \(\tilde g_0\) and the iterative solution of Eq.~\eqref{proxyopt} are computed with ZO estimators~\citep{malladi2023mezo}; consequently, Step~2 does not require BP or access to exact gradients. Algorithm~\ref{alg:zoif} presents the full-objective form for clarity, while the implementation uses stochastic minibatch approximations of the same local objective. The theoretical stationarity residual in Theorem~\ref{theorembound} records the combined error of this finite stochastic solve.

Step~2 also admits alternative ZO-compatible approximations of the inverse-Hessian action.\footnote{Detailed formulations and score-error extensions for the plug-in Step~2 variants are provided in the supplementary material.}

\subsubsection{Step 3: ZO Directional Influence Score Estimation}

Following the ZO paradigm, where gradients are approximated using only loss evaluations, we retain the core ZO rationale but adopt a more efficient strategy that bypasses explicit per-sample gradient reconstruction. Specifically, we directly approximate the directional derivative \(\nabla_{\boldsymbol{\theta}} \ell(\boldsymbol{\theta}_0, z_i)^\top \Delta_{\boldsymbol{\theta}}\) via a forward finite difference along the previously computed parameter update direction \(\Delta_{\boldsymbol{\theta}}\):

\begin{equation}
w_i = -m_p \cdot \frac{\ell(\boldsymbol{\theta}_0 + \alpha \Delta_{\boldsymbol{\theta}}, z_i) - \ell(\boldsymbol{\theta}_0, z_i)}{\alpha}\label{step3},
\end{equation}

which converges to \(-m_p \, \nabla_{\boldsymbol{\theta}} \ell(\boldsymbol{\theta}_0, z_i)^\top \Delta_{\boldsymbol{\theta}}\) as \(\alpha \to 0\). 
This formulation requires two forward passes per sample and avoids the backward pass and per-sample gradient storage required by \texttt{FOIF}.

Based on this, we rank the samples in \(\mathcal{D}_{\text{pool}}\) by their influence scores in ascending order and select the top \(n_s\) samples with the smallest influence for subsequent fine-tuning.

\begin{algorithm}[tb]
\caption{\texttt{ZOIF}: Zeroth-Order Influence Function Estimation}\label{alg:zoif}
\begin{algorithmic}[1]
\footnotesize
\setlength{\itemsep}{1pt}
\REQUIRE $\mathcal{D}_{\text {pool }}$, $\mathcal{D}_{\text {proxy}}$, LLM model initial parameters $\boldsymbol{\theta}_0$ and loss function $\ell$, proxy-set size $m_p$, ZO parameters $(K,\alpha)$, size $|\mathcal{S}|$.
\ENSURE Estimated influence scores $\{w_i\}_{i=1}^n$.

\Algphase{Step 1: Average ZO Gradient Estimation}
\STATE Randomly select a subset \(\mathcal{S}\subseteq\mathcal{D}_{\text{proxy}}\) of size \(|\mathcal{S}|\)
\STATE Compute average ZO gradient estimate:
\STATE \(\displaystyle \tilde{g}_{\text{proxy}}=\frac{1}{|\mathcal{S}|}\sum_{z'_j\in\mathcal{S}}\texttt{ZOGE}(\ell,\boldsymbol{\theta}_0,z'_j,K,\alpha)\)

\Algphase{Step 2: Solving the IHVP via ZO Optimization}
\STATE Compute the ZO reference-gradient estimate $\tilde g_0\gets\texttt{ZOGE}(\mathcal{L},\boldsymbol{\theta}_0)$
\STATE Define auxiliary optimization problem as Eq.~\eqref{proxyopt}
\STATE Solve for $\boldsymbol{\theta}'$ using a ZO iterative solver and project the update onto the prescribed norm ball
\STATE Get $\Delta_{\boldsymbol{\theta}}\gets\boldsymbol{\theta}'-\boldsymbol{\theta}_0$

\Algphase{Step 3: ZO Directional Influence Score Estimation}
\FOR{each data pool sample $z_i$}
    \STATE Compute: $\ell_0\gets\ell(\boldsymbol{\theta}_0,z_i)$
    \STATE Compute: $\ell^+\gets\ell(\boldsymbol{\theta}_0+\alpha\Delta_{\boldsymbol{\theta}},z_i)$
    \STATE Estimate influence score: $w_i\gets-m_p\cdot\dfrac{\ell^+-\ell_0}{\alpha}$
\ENDFOR
\STATE \textbf{Return} $\{w_i\}_{i=1}^n$

\end{algorithmic}
\end{algorithm}

\begin{remark}
For full-parameter LLMs with batch size \(B=1\), relative to \texttt{FOIF}, \texttt{ZOIF} saves \(A_{\mathrm{activations}}-A_{\max}\) peak memory in Step~1, at least the same amount per gradient evaluation in Step~2, and \(d\beta+A_{\mathrm{activations}}-A_{\max}\) in Step~3.
Here, \(A_{\mathrm{activations}}=L(HS^2+SD_{\mathrm{ffn}})\beta\) approximates cumulative BP activation memory, and \(A_{\max}\) is the largest single-layer activation retained by a forward evaluation; full notation and derivations are given in the supplementary material.

\end{remark}

\subsection{Theoretical Properties}

This section provides the theoretical analysis for the proposed \texttt{ZOIF}. We establish upper bounds on the deviation between the influence scores estimated by \texttt{ZOIF} and \texttt{FOIF} under two settings. First, we consider a locally invertible Hessian in the full parameter space. We then relax this condition to a low-rank Hessian structure that better reflects the local geometry of LLMs. All proofs are provided in the supplementary material.

\subsubsection{Deviation Between \texttt{ZOIF} and \texttt{FOIF}}

\begin{assumption}\label{hessian_bounded}
There exists \(M>0\) such that the per-sample Hessian norm is at most \(M\) along every ZO perturbation segment used in Steps~1--2 and every Step~3 segment considered below.
\end{assumption}

\begin{assumption}\label{assusgradb}
There exists \(G>0\) such that \(\|\nabla\ell(\boldsymbol{\theta}_0,z)\|_2\leq G\) for every proxy, reference, and scored sample \(z\).
\end{assumption}

\begin{assumption}\label{h_inv_bounded}
Along the local Step~2 segment \(\{\boldsymbol{\theta}_0+t\Delta_{\boldsymbol{\theta}}:t\in[0,1]\}\), \(\nabla^2\mathcal{L}(\boldsymbol{\theta})\succeq\mu\boldsymbol{I}\) for some \(\mu>0\). In particular, \(\|\boldsymbol{H}_{\boldsymbol{\theta}_0}^{-1}\|_2\leq1/\mu\), and the integrated Hessian
\[
\overline{\boldsymbol{H}}(\Delta_{\boldsymbol{\theta}})
=\int_0^1\nabla^2\mathcal{L}(\boldsymbol{\theta}_0+t\Delta_{\boldsymbol{\theta}})\,dt
\]
also has inverse norm at most \(1/\mu\).
\end{assumption}

\begin{assumption}\label{hessian_lip}
The per-sample Hessians on the above ZO and scoring segments, and the Hessian of \(\mathcal{L}\) on the local Step~2 segment, are \(L_H\)-Lipschitz.
\end{assumption}

Let \(\tau\) bound the first-order residual of the finite Step~2 output:
\[
\|\nabla\mathcal{L}(\boldsymbol{\theta}_0+\Delta_{\boldsymbol{\theta}})
-\widetilde{\boldsymbol{g}}_0-\widetilde{\boldsymbol{g}}_{\mathrm{proxy}}\|_2
\leq\tau.
\]
The case \(\tau=0\) is a local stationary solution; no global minimizer is required.

\begin{theorem}\label{theorembound}
Under Assumptions~\ref{hessian_bounded}--\ref{hessian_lip}, suppose the Step~2 norm projection is inactive. For each scored sample \(z_i\), let \(\mathcal{I}_{\mathrm{FO}}(z_i)\) denote its \texttt{FOIF} score. Conditional on the sampled subsets, for any \(0<\delta_0<1\), with probability at least \(1-\delta_0\) over the Gaussian directions, the outputs \(w_i\) of Algorithm~\ref{alg:zoif} satisfy, simultaneously for all scored samples \(z_i\),
\begin{align*}
&|\mathcal{I}_{\mathrm{FO}}(z_i)
-w_i|\\
&\leq \frac{m_pGE_\tau}{\mu}
+\frac{m_pL_HG^2(G+E_\tau)}{2\mu^3}\\
&\quad+\frac{m_pM\alpha}{2}
\left(\frac{G+E_\tau}{\mu}\right)^2\\
&\quad+\frac{m_pL_H\alpha^2}{6}
\left(\frac{G+E_\tau}{\mu}\right)^3.
\end{align*}
\end{theorem}

where
\begin{align*}
E_\tau
&=\varepsilon_{\mathrm{sub}}+\varepsilon_{\mathrm{ref}}+\tau\\
&\quad+\mathcal E_d(q_p,K,\alpha,\delta_0/2)
+\mathcal E_d(q_0,K,\alpha,\delta_0/2),
\end{align*}
and
\begin{align*}
\mathcal E_d(q,K,\alpha,\delta)
&=G\sqrt{\frac{d+1}{qK\delta}}\\
&\quad+\frac{\alpha M}{qK}\sum_{j=1}^{q}\sum_{k=1}^{K}
\|\boldsymbol u_{jk}\|_2^3\\
&\quad+\frac{2L_H\alpha^2}{3qK}
\sum_{j=1}^{q}\sum_{k=1}^{K}\|\boldsymbol u_{jk}\|_2^4.
\end{align*}
Here, \(q_p\), \(q_0\), \(\varepsilon_{\mathrm{sub}}\), and \(\varepsilon_{\mathrm{ref}}\) are defined in the supplementary material.
Thus, increasing \(K\) reduces the Gaussian sketching error as \(O(K^{-1/2})\), while decreasing \(\alpha\) reduces the finite-radius and Step~3 errors, which scale as \(O(\alpha)+O(\alpha^2)\).

As a ranking consequence, if the right-hand side of Theorem~\ref{theorembound} is uniformly at most \(\varepsilon_{\mathrm{score}}\), the ascending Top-\(n_s\) set is preserved when the \texttt{FOIF} boundary gap exceeds \(2\varepsilon_{\mathrm{score}}\).

\subsubsection{Deviation under Low-Rank Hessian Structure}

While Theorem~\ref{theorembound} assumes positive curvature in the full parameter space, pre-trained LLMs can exhibit highly anisotropic local curvature, with most Hessian eigenvalues concentrated near zero and only a small number of dominant directions~\citep{malladi2023mezo}. The full-space condition can therefore be restrictive. To address this, we replace it with a low-rank non-degeneracy condition and analyze the corresponding projected \texttt{FOIF} target.

\begin{assumption}\label{lowrank_hessian}
There is an orthonormal basis \(\boldsymbol{U}_r\in\mathbb{R}^{d\times r}\) such that Step~2 and all ZO perturbations are restricted to \(\operatorname{span}(\boldsymbol{U}_r)\). The projected reference and integrated Hessians satisfy the local curvature condition in Assumption~\ref{h_inv_bounded}. Define
\[
\boldsymbol{H}_{\boldsymbol{\theta}_0}^{\dagger}
=\boldsymbol{U}_r(\boldsymbol{U}_r^\top\boldsymbol{H}_{\boldsymbol{\theta}_0}\boldsymbol{U}_r)^{-1}\boldsymbol{U}_r^\top.
\]
\end{assumption}

Based on this assumption, the following theorem characterizes the estimation error in the low-rank setting.

\begin{theorem}\label{theoremboundLowrH}
Under Assumptions~\ref{hessian_bounded}--\ref{hessian_lip} and~\ref{lowrank_hessian}, suppose the Step~2 norm projection is inactive. Let \(\mathcal{I}_{\mathrm{FO}}^{(r)}(z_i)\) denote the projected \texttt{FOIF} score obtained using \(\boldsymbol{H}_{\boldsymbol{\theta}_0}^{\dagger}\). Conditional on the sampled subsets, for any \(0<\delta_0<1\), with probability at least \(1-\delta_0\) over the Gaussian directions, simultaneously for all scored samples \(z_i\),
\begin{align*}
|\mathcal{I}_{\mathrm{FO}}^{(r)}(z_i)-w_i|
&\leq \frac{m_pGE_{r,\tau}}{\mu}\\
&\quad+\frac{m_pL_HG^2(G+E_{r,\tau})}{2\mu^3}\\
&\quad+\frac{m_pM\alpha}{2}
\left(\frac{G+E_{r,\tau}}{\mu}\right)^2
\\
&\quad+\frac{m_pL_H\alpha^2}{6}
\left(\frac{G+E_{r,\tau}}{\mu}\right)^3,
\end{align*}
where
\begin{align*}
E_{r,\tau}
&=\varepsilon_{\mathrm{sub}}+\varepsilon_{\mathrm{ref}}+\tau\\
&\quad+\mathcal E_r(q_p,K,\alpha,\delta_0/2)\\
&\quad+\mathcal E_r(q_0,K,\alpha,\delta_0/2),
\end{align*}
and \(\mathcal E_r\) is the subspace counterpart of \(\mathcal E_d\), evaluated with rank \(r\) and perturbations restricted to \(\operatorname{span}(\boldsymbol U_r)\).
\end{theorem}

Thus, under the explicit rank-\(r\) restriction, the sketching error depends on the effective dimension \(r\) rather than the ambient dimension \(d\), and the deviation from the projected \texttt{FOIF} target remains controlled when \(E_{r,\tau}\) and the finite-difference terms are small.

\section{Experiment Study}\label{experiment}

\begin{table*}[!t]
\centering
{\footnotesize
\setlength{\tabcolsep}{2.3pt}
\begin{tabular}{@{}l *{10}{c} l@{}}
\toprule
\multicolumn{12}{c}{\textbf{Peak GPU memory (GB)}}\\
\cmidrule(lr){1-12}
\multirow{3}{*}{LLM} & \multicolumn{5}{c}{\texttt{FOIF}} & \multicolumn{5}{c}{\texttt{ZOIF}} & \multirow{3}{*}{Mode}\\
\cmidrule(lr){2-6}\cmidrule(lr){7-11}
& S1 & \texttt{dia} & \texttt{hvp} & \texttt{opt} & S3
& S1 & \texttt{dia} & \texttt{hvp} & \texttt{opt} & S3 &\\
\midrule
\texttt{pythia-70m}  & 0.68 & 0.94 & 0.86 & 2.70 & 0.92 & 0.63 & 0.90 & 0.63 & 1.72 & 0.51 & Full\\
\texttt{Llama-3.2-1B}& 3.91 & 3.94 & 3.96 & 10.67 & 3.91 & 2.96 & 2.97 & 2.97 & 3.00 & 2.95 & LoRA\\
\texttt{Llama-2-7B}  & 15.74 & 15.56 & 15.62 & 34.39 & 15.46 & 13.08 & 13.15 & 13.15 & 13.27 & 13.02 & LoRA\\
\texttt{Qwen3-0.6B}  & 3.09 & 3.11 & 3.13 & 15.77 & 3.10 & 1.92 & 1.94 & 1.94 & 1.98 & 1.91 & LoRA\\
\texttt{Qwen3-1.7B}  & 5.68 & 5.73 & 5.74 & 21.46 & 5.69 & 4.03 & 4.06 & 4.06 & 4.11 & 4.01 & LoRA\\
\addlinespace[3pt]
\multicolumn{12}{c}{\textbf{Runtime (hours)}}\\
\cmidrule(lr){1-12}
\multirow{3}{*}{LLM} & \multicolumn{5}{c}{\texttt{FOIF}} & \multicolumn{5}{c}{\texttt{ZOIF}} & \multirow{3}{*}{Mode}\\
\cmidrule(lr){2-6}\cmidrule(lr){7-11}
& S1 & \texttt{dia} & \texttt{hvp} & \texttt{opt} & S3
& S1 & \texttt{dia} & \texttt{hvp} & \texttt{opt} & S3 &\\
\midrule
\texttt{pythia-70m}  & 0.017 & 0.021 & 0.040 & 0.033 & 37.60 & 0.27 & 0.33 & 0.58 & 0.60 & 22.50 & Full\\
\texttt{Llama-3.2-1B}& 0.058 & 0.083 & 0.175 & 0.050 & 34.00 & 0.80 & 0.92 & 2.28 & 0.80 & 30.00 & LoRA\\
\texttt{Llama-2-7B}  & 0.100 & 0.150 & 0.320 & 0.050 & 43.00 & 2.00 & 2.50 & 5.00 & 2.00 & 40.00 & LoRA\\
\texttt{Qwen3-0.6B}  & 0.100 & 0.150 & 0.310 & 0.050 & 40.00 & 1.50 & 1.70 & 3.50 & 1.00 & 31.75 & LoRA\\
\texttt{Qwen3-1.7B}  & 0.100 & 0.150 & 0.320 & 0.050 & 40.30 & 1.50 & 1.80 & 3.50 & 1.60 & 32.00 & LoRA\\
\bottomrule
\end{tabular}
}
\caption{Step-wise peak GPU memory and runtime for \texttt{FOIF} and \texttt{ZOIF}. S1 and S3 denote Steps~1 and~3; the \texttt{dia}, \texttt{hvp}, and \texttt{opt} columns report the corresponding Step~2 variants.}
\label{tab:efficiency_all_steps}
\end{table*}

\subsection{Experimental Design}

\subsubsection{Data Pool and Validation Data}

Following the experimental setup of~\citet{ivison2023camels,ivison2025large}, we randomly sample a subset of 200,000 examples from the TÜLU2 dataset to form the candidate data pool. From this pool, we further select 10,000 samples for subsequent model fine-tuning.
To ensure fairness and comparability, we evaluate our approach on three widely used benchmark datasets: MMLU~\citep{hendryckstest2021}, BBH~\citep{suzgun2022challenging}, and SQuAD~\citep{rajpurkar-etal-2016-squad}. Detailed evaluation protocols are provided in the supplementary material.

Because the evaluation spans multiple tasks, we consider two data selection strategies. \textbf{Strategy 1} aggregates influence scores across tasks to produce a single ranking for joint fine-tuning and evaluation. \textbf{Strategy 2} ranks and selects data separately for each task, followed by task-specific fine-tuning and evaluation.

\subsubsection{LLMs}
We use full-parameter influence computation and fine-tuning for \texttt{pythia-70m}~\citep{biderman2023pythia}. For \texttt{Llama-3.2-1B}~\citep{dubey2024llama}, \texttt{Llama-2-7B}~\citep{touvron2023llama}, \texttt{Qwen3-0.6B}, and \texttt{Qwen3-1.7B}~\citep{qwen3technicalreport}, we perform both influence estimation and fine-tuning in the LoRA parameter subspace. Full parameter settings are provided in the supplementary material.

\subsubsection{Data Selection Methods}

\paragraph{Our Methods}

We evaluate the optimization-based method \texttt{ZOIF\_opt} together with three plug-in variants:
\texttt{ZOIF\_hvp} uses a DataInf-style damped block-wise empirical-Fisher inverse approximation~\citep{Kwonetal2023};
\texttt{ZOIF\_dia} uses a ZO-compatible diagonal preconditioner; and
\texttt{ZOIF\_ide} replaces the curvature operator with the identity. All gradient-dependent quantities in these variants are supplied by ZO estimates. Their unified mathematical form and variant-specific residuals are given in the supplementary material.

For ZO gradient estimation via perturbation, we set \(K = 30\) and \(\alpha = 1\text{e-}3\).\footnote{Sensitivity analyses of $K$, $\alpha$, and the data selection ratio, together with the ranking correlation between \texttt{ZOIF} and \texttt{FOIF}, are provided in the supplementary material.}

\vspace{-0.5ex}
\paragraph{Baseline Methods}  
We compare against: (1) \texttt{FOIF\_opt}, \texttt{FOIF\_hvp}, \texttt{FOIF\_dia}, and \texttt{FOIF\_ide}, the first-order counterparts of the corresponding \texttt{ZOIF} variants; (2) \texttt{LESS}~\citep{xia2024less} and \texttt{TRAK}~\citep{park2023trak}, gradient-based attribution methods using projection or sketching; (3) \texttt{EMB}, based on embedding similarity~\citep{ni2022large,lee2024nv}; (4) \texttt{Top\_PPL} and \texttt{Mid\_PPL}, based on perplexity~\citep{yin2024compute}; (5) \texttt{Random}, which uniformly selects 10,000 examples; and (6) \texttt{Orig}, the original model without fine-tuning.
All experiments are conducted in the same environment, with results averaged over three independent runs.

\vspace{-0.5ex}
\subsection{Experimental Results}

\subsubsection{Memory and Runtime Efficiency}

\paragraph{Memory Efficiency}  
Table~\ref{tab:efficiency_all_steps} reports peak GPU memory across all three steps. \texttt{ZOIF} uses less memory than its \texttt{FOIF} counterpart in every evaluated setting, with the largest reductions appearing in the Step~2 optimization variant.
\footnote{
We further compare \texttt{ZOIF} against \texttt{FOIF} augmented with standard memory optimization techniques (Gradient Checkpointing and Activation Offloading) in Step~1 on \texttt{Llama-3.2-1B}. Results show that even when \texttt{FOIF} employs both optimizations simultaneously in this step, \texttt{ZOIF} still retains a clear memory advantage; see the supplementary material for details.}

\paragraph{Time Efficiency}  
Table~\ref{tab:efficiency_all_steps} also provides the step-wise runtime. ZO perturbations increase the cost of Steps~1 and~2, whereas the two-forward-pass score in Step~3 is faster than FO scoring. Because Step~3 dominates the end-to-end workload, the overall runtime remains comparable across the evaluated models.

\subsubsection{Data Selection Performance}  
We report a representative result below; complete evaluations are provided in the supplementary material.

\paragraph{Results on \texttt{Qwen}}  

Results for \texttt{Qwen3-0.6B} are presented in Table~\ref{tab:Qwen3-0.6B_results}.  
On this model, \texttt{ZOIF\_dia} achieves the highest average score (51.97), compared with \texttt{LESS} (51.68), \texttt{TRAK} (50.61), and the evaluated \texttt{FOIF} variants. Task-dependent rankings among the curvature variants motivate task-specific operator selection. Additional model results, sensitivity analyses, and stability tests are reported in the supplement.

\begin{table}[htbp]
\centering
{\small
\setlength{\tabcolsep}{3pt}
\begin{tabular}{@{}lcccc@{}}
\toprule
\textbf{Method} & \textbf{MMLU} & \textbf{BBH} & \textbf{SQuAD} & \textbf{Avg.} \\
\midrule
  \texttt{ZOIF\_dia} & 44.09 & 34.67 & 77.17 & \best{51.97} $\pm$ 0.32 \\
  \texttt{ZOIF\_hvp} & 43.89 & 34.38 & 74.70 & 50.99 $\pm$ 0.96 \\
  \texttt{ZOIF\_ide} & 44.50 & 34.32 & 76.30 & 51.71 $\pm$ 0.44 \\
  \texttt{ZOIF\_opt} & 44.41 & \best{34.81} & 73.42 & 50.88 $\pm$ 0.76 \\
  \texttt{FOIF\_dia} & 43.95 & 34.19 & 72.38 & 50.17 $\pm$ 1.02 \\
  \texttt{FOIF\_hvp} & 43.89 & 33.77 & 75.21 & 50.96 $\pm$ 0.65 \\
  \texttt{FOIF\_ide} & 44.16 & 32.23 & 67.20 & 47.86 $\pm$ 1.77 \\
  \texttt{FOIF\_opt} & 45.22 & 34.08 & 72.61 & 50.64 $\pm$ 1.12 \\
  \texttt{LESS} & 43.22 & 33.37 & \best{78.45} & 51.68 $\pm$ 0.23 \\
  \texttt{TRAK} & 42.97 & 32.83 & 76.02 & 50.61 $\pm$ 0.84 \\
  \texttt{EMB} & 45.30 & 33.72 & 70.91 & 49.98 $\pm$ 0.52 \\
  \texttt{Top\_PPL} & 45.75 & 34.00 & 71.52 & 50.43 $\pm$ 0.27 \\
  \texttt{Mid\_PPL} & 45.82 & 33.83 & 72.31 & 50.66 $\pm$ 0.24 \\
  \texttt{Random} & 45.27 & 34.33 & 69.51 & 49.70 $\pm$ 0.37 \\
  \texttt{Orig} & \best{47.09} & 30.54 & 48.45 & 42.03 \\
\bottomrule
\end{tabular}
}
\caption{Held-out \texttt{Qwen3-0.6B} performance (Strategy~1). Avg. is the macro-average of MMLU/BBH accuracy and SQuAD F1. The Avg. entry is reported as mean $\pm$ standard deviation over three independent runs.}
\label{tab:Qwen3-0.6B_results}
\end{table}

\section{Conclusion}\label{conclusion}
\texttt{ZOIF} is a forward-only framework for influence-based data selection under limited GPU memory. Its analysis separates numerical, curvature-operator, and directional finite-difference errors and extends to low-rank subspaces. Across full-parameter and LoRA experiments on five language models, it substantially reduces peak memory while retaining competitive subset utility, offering a practical time--memory trade-off.

\bibliography{example_paper}

\clearpage
\appendix

\section{ZOGE Algorithm}\label{algZOGE}

The \texttt{ZOGE} procedure follows~\citet{malladi2023mezo} and is summarized in Algorithm~\ref{alg:zeroth_order_grad}.

\begin{algorithm}[p]
\caption{\texttt{ZOGE}: Zeroth-Order Gradient Estimation}
\label{alg:zeroth_order_grad}
\small
\begin{algorithmic}[1]
\REQUIRE
    LLM model with parameters $\boldsymbol{\theta}_0 = \{\theta_j\}_{j=1}^d$, loss function $\ell$, input sample $z$, number of perturbations $K$, perturbation radius $\alpha > 0$.
\ENSURE Estimated gradient $\tilde{\nabla}_{\boldsymbol{\theta}} \ell(\boldsymbol{\theta}_0, z)$.

\STATE Initialize estimate: $\tilde{\nabla}_{\boldsymbol{\theta}} \ell(\boldsymbol{\theta}_0, z) \gets \mathbf{0}$ (same shape as $\boldsymbol{\theta}_0$)
\STATE Initialize valid counter: $C_{\text{valid}} \gets 0$
\STATE $\text{coef},\text{seeds}\gets[]$
\STATE Compute baseline loss: $\ell_0 \gets \ell(\boldsymbol{\theta}_0, z)$
\IF{$\ell_0$ is \texttt{NaN} or \texttt{Inf}}
    \STATE \textbf{return} $\mathbf{0}$  \COMMENT{or skip this sample in outer loop}
\ENDIF
\FOR{$i = 1, 2, \ldots, K$}
    \STATE Sample a random seed $s_i$
    \STATE Set random number generator seed to $s_i$
    \FOR{each parameter $\theta_j \in \boldsymbol{\theta}_0$}
        \STATE Sample scalar $u_j \sim \mathcal{N}(0, 1)$
        \STATE Perturb in-place: $\theta_j \gets \theta_j + 2 \alpha \cdot u_j$
    \ENDFOR
    \STATE Compute $\ell^+ \gets \ell(\boldsymbol{\theta}_0, z)$
    \STATE Reset random number generator seed to $s_i$
    \FOR{each parameter $\theta_j \in \boldsymbol{\theta}_0$}
        \STATE Sample scalar $u_j \sim \mathcal{N}(0, 1)$
        \STATE Restore: $\theta_j \gets \theta_j - 2\alpha \cdot u_j$
    \ENDFOR
    \IF{$\ell^+$ is \texttt{NaN} or \texttt{Inf}}
        \STATE \textbf{continue}
    \ENDIF
    \STATE Compute coefficient: $c_i \gets \dfrac{\ell^+ - \ell_0}{2\alpha}$    
    \STATE Append $s_i$ to $\text{seeds}$   
    \STATE Append $c_i$ to $\text{coef}$
    \STATE $C_{\text{valid}} \gets C_{\text{valid}} + 1$
\ENDFOR
\FOR{$i = 1, 2, \ldots, C_{\text{valid}}$}
    \STATE Reset random number generator seed to $\text{seeds}[i]$
    \STATE $c_i=\text{coef}[i]$
    \FOR{each parameter $\theta_j \in \boldsymbol{\theta}_0$}
        \STATE Sample scalar $u_j \sim \mathcal{N}(0, 1)$
        \STATE $(\tilde{\nabla}_{\boldsymbol{\theta}} \ell(\boldsymbol{\theta}_0, z))_j \gets (\tilde{\nabla}_{\boldsymbol{\theta}} \ell(\boldsymbol{\theta}_0, z))_j + c_i \cdot u_j$
    \ENDFOR
\ENDFOR
\IF{$C_{\text{valid}} > 0$}
    \STATE $\tilde{\nabla}_{\boldsymbol{\theta}} \ell(\boldsymbol{\theta}_0, z) \gets \dfrac{1}{C_{\text{valid}}} \tilde{\nabla}_{\boldsymbol{\theta}} \ell(\boldsymbol{\theta}_0, z)$
\ELSE
    \STATE $\tilde{\nabla}_{\boldsymbol{\theta}} \ell(\boldsymbol{\theta}_0, z) \gets \mathbf{0}$
\ENDIF
\STATE \textbf{Return} $\tilde{\nabla}_{\boldsymbol{\theta}} \ell(\boldsymbol{\theta}_0, z)$
\end{algorithmic}
\end{algorithm}
\section{Memory and Computational Complexity}
\subsection{Memory Analysis}\label{memoryana}

To compare the memory overhead of \texttt{ZOIF} and \texttt{FOIF}, we use the following variables for model architecture and computation:
\begin{itemize}
    \item $d$: total number of model parameters (e.g., 7B, 13B);
    \item $S$: input sequence length (in tokens);
    \item $L$: number of Transformer layers;
    \item $H$: number of attention heads;
    \item $D_{\text{model}}$: hidden dimension;
    \item $D_{\text{ffn}} \approx 4 D_{\text{model}}$: intermediate dimension of the feed-forward network (FFN);
    \item $B$: batch size;
    \item $\beta \in \{2, 4\}$: bytes per parameter (2 for bfloat16/float16, 4 for float32).
\end{itemize}

\subsubsection{Step 1}

In Step 1, we estimate the average gradient over the subset drawn from the proxy dataset. For full-parameter computation with $B=1$, both methods require loading the model parameters ($M_{\text{params}}$), computing the gradient for a single sample ($G_{\text{grads}}$), and accumulating the total gradient in a gradient pool ($G_{\text{gradspool}}$).

For FO methods, the peak memory consumption is
\[
\text{Mem}_1 = M_{\text{params}} + G_{\text{grads}} + A_{\text{activations}} + G_{\text{gradspool}},
\]
where $M_{\text{params}} = G_{\text{grads}} = G_{\text{gradspool}} = d\beta$. The term $A_{\text{activations}}$ accounts for all intermediate tensors required to support backpropagation, including both the computational graph metadata and, dominantly, the cached activations from each layer. In Transformers, these activations are primarily the self-attention scores ($O(H S^2)$) and FFN intermediate outputs ($O(S D_{\text{ffn}})$), leading to
\[
A_{\text{activations}} \approx L (H S^2 + S D_{\text{ffn}}) \beta.
\]
Thus, the total memory usage is
\[
\text{Mem}_1(S) = 3 d \beta + L (H S^2 + S D_{\text{ffn}}) \beta.
\]

In contrast, ZO methods estimate gradients by perturbing model parameters and measuring the resulting change in loss. Following recent works (e.g.,~\citet{malladi2023mezo,zhang2025towards}) and as illustrated in Algorithm~\ref{alg:zeroth_order_grad}, advanced ZO strategies avoid storing full perturbation vectors. Instead, only random seeds are retained; during parameter updates, a scalar Gaussian noise $z \sim \mathcal{N}(0,1)$ is independently sampled for each parameter $\theta_i$, and the gradient projection is applied directly. The final gradient estimate is reconstructed on-the-fly using the stored seeds, eliminating the need to store high-dimensional perturbation directions.
This approach incurs only negligible overhead for storing \(K\) random seeds and scalar coefficients (\(O(K)\), typically \(K \ll d\)); no full perturbation vector is materialized. The reconstruction therefore uses parameter-sized storage only for the gradient estimate itself.

Since these computations are performed within a \texttt{torch.no\_grad()} context—without building a computational graph—the memory footprint is
\[
\text{Mem}_0 = M_{\text{params}} + G_{\text{est}} + A_{\max} + G_{\text{gradspool}},
\]
where $A_{\max}$ denotes the activation memory of the largest layer. Consequently, the ZO peak memory does not accumulate activations across layers:
\[
\text{Mem}_0 = 3 d \beta + A_{\max}.
\]
For the evaluated configurations, $A_{\max}$ is smaller than the cumulative BP activation term.

A direct comparison shows that FO methods incur an additional activation memory term absent in ZO, implying $\text{Mem}_0 < \text{Mem}_1$ for all $S > 0$.

If gradient checkpointing is enabled for FO, the activation memory can be reduced from $O(L)$ to $O(\sqrt{L})$, yielding
\[
\text{Mem}_1^{\text{ckpt}}(S) = 3 d \beta + \sqrt{L} (H S^2 + S D_{\text{ffn}}) \beta.
\]
In this case, the memory advantage of ZO is somewhat diminished but still persists for sufficiently long sequences.

When FO uses a batch size $B > 1$, activation memory scales linearly with $B$, further increasing its memory footprint.

Under LoRA adaptation, both methods load the same base model parameters, but all gradient-related terms ($G_{\text{grads}}, G_{\text{gradspool}}$) scale with the LoRA rank rather than $d$. Nevertheless, FO still requires storing activations and maintaining the computational graph for backpropagation, resulting in higher memory usage than ZO.

\subsubsection{Step 2}

In Step 2, both \texttt{FOIF} and \texttt{ZOIF} load the model parameters and the direction computed from Step 1. With an iterative optimizer, each FO gradient evaluation retains the BP graph, whereas its ZO counterpart uses transient forward activations. The peak-memory saving per gradient evaluation is therefore at least $A_{\text{activations}}-A_{\max}$. Sequential evaluations add computation but do not, by themselves, multiply peak memory; implementations that retain higher-order graphs may incur additional FO memory.
As \(B\) increases, both methods process more transient activations, while \texttt{FOIF} additionally retains the layer-wise activations required for BP.

If FO employs adaptive optimizers (e.g., Adam), additional optimizer states must be stored, further increasing memory consumption. These observations are consistent with findings in~\citet{malladi2023mezo,zhang2025towards}.

Alternatively, if the IHVP is approximated via a diagonal Hessian instead of an optimizer, the approximation still relies on gradient estimation. In this case, FO again requires building a computational graph and caching activations, leading to higher memory usage than ZO.

\subsubsection{Step 3}

Step 3 exhibits the largest memory disparity between \texttt{FOIF} and \texttt{ZOIF}.

Both methods share the cost of loading model parameters and the parameter updates from Step 2. The key difference lies in how the influence score is computed:
\begin{itemize}
    \item \texttt{FOIF} computes a separate gradient for each target sample, which entails storing activations, maintaining the computational graph, and caching gradient vectors.
    \item \texttt{ZOIF} requires only two forward passes per sample to estimate the influence, completely avoiding gradient computation and storage. It incurs only the activation memory of a single forward pass, i.e., $A_{\max}$.
\end{itemize}
Combining the insights from Step 1, the per-sample memory saving is  
$d\beta + A_{\text{activations}} - A_{\max}.$

In summary, \texttt{ZOIF} consistently incurs lower memory overhead than \texttt{FOIF}.

\subsection{Computational Complexity Analysis}
\label{app:complexity-analysis}

We provide a comprehensive analysis of the computational complexity (FLOPs) for both \texttt{FOIF} and \texttt{ZOIF}.

\subsubsection{Theoretical Complexity Framework}
We estimate FLOPs based on the cost of a single forward pass ($C_{\text{fwd}}$), where $C_{\text{fwd}} \approx 2 \cdot d \cdot S$ (with $d$ as the number of parameters and $S$ as the sequence length). Following standard conventions in LLM training, one backward pass costs approximately twice a forward pass ($C_{\text{bwd}} \approx 2C_{\text{fwd}}$), making the total cost of a first-order gradient update $C_{\text{grad}} \approx 3C_{\text{fwd}}$. Table~\ref{tab:flops_comparison} summarizes the complexity for each computational step, with $K=30$ perturbations for \texttt{ZOIF} in our experiments.

\begin{table}[tb]
\centering
{\small
\setlength{\tabcolsep}{3pt}
\begin{tabular}{@{}lccc@{}}
\toprule
\textbf{Step} & \textbf{FOIF Cost} & \textbf{\texttt{ZOIF} Cost} & \textbf{Ratio (ZO/FO)} \\
\midrule
Step 1 & $N_{\text{prox}} \cdot 3C_{\text{fwd}}$ & $N_{\text{prox}} \cdot (K+1)C_{\text{fwd}}$ & $\approx 10.3\times$ \\
Step 2 & $T B \cdot 3C_{\text{fwd}}$ & $T B \cdot (K+1)C_{\text{fwd}}$ & $\approx 10.3\times$ \\
Step 3 & $N_{\text{pool}} \cdot 3C_{\text{fwd}}$ & $N_{\text{pool}} \cdot 2C_{\text{fwd}}$ & $\approx 0.67\times$ \\
\bottomrule
\end{tabular}
}
\caption{Computational Complexity (FLOPs) Comparison ($K=30$)}
\label{tab:flops_comparison}
\end{table}

\subsubsection{Step-wise Analysis.}
In Steps 1 and 2 (initialization stages), \texttt{ZOIF} requires $K$ perturbations to estimate gradient and Hessian-vector product (HVP) directions. With $K=30$, each iteration involves 31 forward passes, making it approximately \textbf{10.3$\times$} more computationally expensive than \texttt{FOIF}'s single backpropagation ($3C_{\text{fwd}}$). This explains the increased runtime during early stages observed in our empirical evaluations.

However, in Step 3 (influence scoring over the large pool $N_{\text{pool}}$), \texttt{ZOIF} demonstrates an efficiency advantage: it requires only two forward passes per sample (for finite-difference loss estimation), whereas \texttt{FOIF} must perform a full backpropagation. Consequently, \texttt{ZOIF} reduces per-sample FLOPs by \textbf{33\%} (from $3C_{\text{fwd}}$ to $2C_{\text{fwd}}$).

\subsubsection{Scaling Behavior.}
Both methods scale linearly with pool size, $O(N_{\text{pool}} \cdot d \cdot S)$. In Step~3, \texttt{ZOIF} has a lower per-sample constant and avoids activation accumulation across layers, which permits larger batches within a fixed memory budget. The complexity of \texttt{ZOIF} in Steps 1 and 2 scales linearly with $K$, while both methods scale linearly with model size $d$.

\subsection{Time-Memory Trade-off}
The theoretical FLOPs analysis reveals a deliberate \textbf{time-for-memory trade-off} inherent to \texttt{ZOIF}. While \texttt{ZOIF} incurs higher FLOPs in initialization steps, this computational overhead is the price paid for eliminating the $O(L)$ or $O(\sqrt{L})$ activation memory required for backpropagation.
This trade-off enables influence-based selection on large models where the corresponding \texttt{FOIF} configuration would encounter Out-of-Memory (OOM) errors on identical hardware, making \texttt{ZOIF} a viable option under tight memory constraints.

\subsection{Difference from Direct ZO Influence Computation}
A direct ZO translation of conventional IF computation would reconstruct a random-direction gradient for every proxy and candidate sample. \texttt{ZOIF} instead exploits the shared structure of aggregate influence. Step~1 averages the proxy-side ZO gradients once, and Step~2 converts this average into a single shared direction \(\widehat{\boldsymbol{v}}\). The cost of these two steps is therefore amortized across the entire candidate pool.

The computational difference is explicit at the score level. A direct per-candidate ZO translation first reconstructs
\begin{align}
\widehat{\boldsymbol{g}}_i^{\mathrm{dir}}
&=\frac{1}{K}\sum_{k=1}^{K}
\frac{\ell(\boldsymbol{\theta}_0+2\alpha\boldsymbol{u}_k,z_i)
-\ell(\boldsymbol{\theta}_0,z_i)}{2\alpha}\boldsymbol{u}_k,
\notag\\
\widehat{\mathcal{I}}_i^{\mathrm{dir}}
&=-m_p(\widehat{\boldsymbol{g}}_i^{\mathrm{dir}})^\top
\widehat{\boldsymbol{v}},
\label{eq:direct-zo-if}
\end{align}
which requires (K+1) forward evaluations and a reconstructed gradient for each candidate. In contrast, \texttt{ZOIF} evaluates only the directional derivative
\begin{equation}
\widehat{\mathcal{I}}_i^{\mathrm{ZOIF}}
=-m_p\frac{\ell(\boldsymbol{\theta}_0+
\alpha\widehat{\boldsymbol{v}},z_i)
-\ell(\boldsymbol{\theta}_0,z_i)}{\alpha}.
\label{eq:shared-direction-score}
\end{equation}
Thus, candidate scoring uses exactly two forward evaluations, is independent of \(K\), and never materializes a per-sample gradient vector.

Table~\ref{tab:forward-influence-comparison} further distinguishes \texttt{ZOIF} from existing forward-oriented data-influence methods. In particular, it identifies where loss-only or ZO computation enters each method. \texttt{ZOIF} applies it throughout all three stages of the decomposed IF computation, whereas the alternatives use forward information only for a specific signal or scoring operation.

\begin{table*}[t]
\centering
\small
\setlength{\tabcolsep}{3pt}
\caption{Comparison of forward-oriented influence and data-valuation methods. The fourth column identifies the stage at which ZO or loss-only computation is used.}
\label{tab:forward-influence-comparison}
\begin{tabular}{@{}p{0.16\textwidth}p{0.23\textwidth}p{0.245\textwidth}p{0.295\textwidth}@{}}
\toprule
Method & Proxy/test-side construction & Candidate-side scoring signal & Placement of ZO/loss-only computation \\
\midrule
Forward-INF~\citep{ko2024mirrored}
& Gradient ascent or continual updating on the test set
& Loss change before and after the test-side update
& Candidate scoring only; proxy/test-side updates still use gradients \\
Z0-Inf~\citep{kokhlikyan2025z0inf}
& Loss values from intermediate training checkpoints
& Candidate loss evolution along the checkpoint trajectory
& Loss-only scoring along an existing training trajectory \\
For-Value~\citep{deng2026forvalue}
& Final-layer representations and token-level prediction errors
& Closed-form alignment of representations and prediction errors
& Forward features and errors used directly for valuation \\
\texttt{ZOIF} (ours)
& Averaged proxy ZO gradient followed by one shared direction
& Directional loss difference in Eq.~\eqref{eq:shared-direction-score}
& All three stages: ZO proxy gradient, ZO Step~2, and forward finite difference \\
\bottomrule
\end{tabular}
\end{table*}

Unlike the alternatives in Table~\ref{tab:forward-influence-comparison}, \texttt{ZOIF} preserves the local Hessian-preconditioned IF target and uses ZO computation in all three stages while amortizing its proxy-side and curvature computation over the entire candidate pool. Its second stage is modular: the default \texttt{ZOIF\_opt} solves the auxiliary problem with a ZO optimizer, while \texttt{ZOIF\_hvp}, \texttt{ZOIF\_dia}, and \texttt{ZOIF\_ide} replace this solve with block-wise empirical-Fisher, diagonal, and identity approximations, respectively; all variants retain ZO estimation in the surrounding stages.

\section{Theoretical Results and Proofs}
\subsection{Notation for the Proofs}

For the proofs, let
\begin{align*}
\boldsymbol{g}_{\mathrm{proxy}}
&=\frac{1}{m_p}\sum_{j=1}^{m_p}
\nabla\ell(\boldsymbol{\theta}_0,z'_j),\\
\boldsymbol{H}_0&=\nabla^2\mathcal{L}(\boldsymbol{\theta}_0),\\
\Delta&=\boldsymbol{\theta}'-\boldsymbol{\theta}_0,
\end{align*}
and define
\begin{align*}
\mathcal{I}_{\mathrm{FO}}(z_i)
&=-m_p\nabla\ell(\boldsymbol{\theta}_0,z_i)^\top
\boldsymbol{H}_0^{-1}\boldsymbol{g}_{\mathrm{proxy}},\\
\widehat{\mathcal{I}}_{\mathrm{ZOIF}}(z_i)
&=w_i,\\
w_i&=-\frac{m_p}{\alpha}\left[
\ell(\boldsymbol{\theta}_0+\alpha\Delta,z_i)
-\ell(\boldsymbol{\theta}_0,z_i)\right].
\end{align*}
For an average over \(q\) samples with \(K\) independent Gaussian directions per sample, define
\begin{align}
\mathcal{E}_d(q,K,\alpha,\delta)
&=G\sqrt{\frac{d+1}{qK\delta}}\notag\\
&\quad+\frac{\alpha M}{qK}\sum_{j=1}^{q}\sum_{k=1}^{K}
\|\boldsymbol{u}_{jk}\|_2^3\notag\\
&\quad+\frac{2L_H\alpha^2}{3qK}\sum_{j=1}^{q}\sum_{k=1}^{K}
\|\boldsymbol{u}_{jk}\|_2^4.
\label{eq:zo-error-supp}
\end{align}
Let \(\boldsymbol{g}_{\mathcal S}\) be the exact proxy-subset average, let \(q_p=|\mathcal{S}|\), and let \(\boldsymbol{g}_{0,q_0}\) be the exact average over the \(q_0\) reference samples used for \(\widetilde{\boldsymbol{g}}_0\). Define
\begin{align*}
\varepsilon_{\mathrm{sub}}
&=\|\boldsymbol{g}_{\mathcal S}-\boldsymbol{g}_{\mathrm{proxy}}\|_2,\\
\varepsilon_{\mathrm{ref}}
&=\|\boldsymbol{g}_{0,q_0}-\nabla\mathcal{L}(\boldsymbol{\theta}_0)\|_2,\\
E_p&=\varepsilon_{\mathrm{sub}}
+\mathcal{E}_d(q_p,K,\alpha,\delta_0/2),\\
E_0&=\varepsilon_{\mathrm{ref}}
+\mathcal{E}_d(q_0,K,\alpha,\delta_0/2),\\
E_\tau&=E_p+E_0+\tau.
\end{align*}
For the low-rank setting, write
\begin{align*}
\boldsymbol{H}_{0,r}
&=\boldsymbol{U}_r^\top\boldsymbol{H}_0\boldsymbol{U}_r,\\
\boldsymbol{H}_0^\dagger
&=\boldsymbol{U}_r\boldsymbol{H}_{0,r}^{-1}\boldsymbol{U}_r^\top,\\
E_{r,\tau}
&=\varepsilon_{\mathrm{sub}}+\varepsilon_{\mathrm{ref}}+\tau\\
&\quad+\mathcal{E}_r(q_p,K,\alpha,\delta_0/2)\\
&\quad+\mathcal{E}_r(q_0,K,\alpha,\delta_0/2),
\end{align*}
where \(\mathcal{E}_r\) is Eq.~\eqref{eq:zo-error-supp} evaluated with \(d\) replaced by \(r\), \(\boldsymbol{u}_{jk}=\boldsymbol{U}_r\boldsymbol{q}_{jk}\), and \(\boldsymbol{q}_{jk}\sim\mathcal{N}(\boldsymbol{0},\boldsymbol{I}_r)\).
\subsection{Proof of Theorem~\ref{theorembound}}
Let \(\boldsymbol{v}_{\mathrm{FO}}=\boldsymbol{H}_0^{-1}\boldsymbol{g}_{\mathrm{proxy}}\) and
\[
d_i(\Delta)=\frac{\ell(\boldsymbol{\theta}_0+\alpha\Delta,z_i)
-\ell(\boldsymbol{\theta}_0,z_i)}{\alpha}.
\]
By the Step~3 definition,
\(\widehat{\mathcal{I}}_{\mathrm{ZOIF}}(z_i)=-m_pd_i(\Delta)\).
Define the directional finite-difference error by
\begin{align*}
\eta_i(\Delta)&=d_i(\Delta)\\
&\quad-\nabla\ell(\boldsymbol{\theta}_0,z_i)^\top\Delta.
\end{align*}
Adding and subtracting the directional derivative along \(\Delta\) gives
\begin{align*}
|\mathcal{I}_{\mathrm{FO}}(z_i)-\widehat{\mathcal{I}}_{\mathrm{ZOIF}}(z_i)|
&\leq m_pG\|\boldsymbol{v}_{\mathrm{FO}}-\Delta\|_2\\
&\quad+m_p|\eta_i(\Delta)|.
\end{align*}
We next decompose the direction error into ZO/solver and local-curvature contributions, and then bound the remaining directional finite-difference error.

\subsubsection{Forward ZO Gradient Error}

The estimator is the forward-difference Gaussian-smoothing construction of~\citet{Nesterov2017Random}, implemented with the random-direction updates used by~\citet{malladi2023mezo}. The following lemma records the finite-sample sketching and finite-radius terms needed here.

\begin{lemma}[Forward ZO gradient error]\label{lemma:zo-gradient}
Under Assumptions~\ref{hessian_bounded}--\ref{hessian_lip}, an average over \(q\) samples estimated with \(K\) independent Gaussian directions per sample satisfies, with probability at least \(1-\delta\),
\[
\|\widetilde{\boldsymbol{g}}-\boldsymbol{g}_{q}\|_2
\leq\mathcal{E}_d(q,K,\alpha,\delta),
\]
where \(\boldsymbol{g}_{q}\) is the corresponding exact sample-average gradient.
\end{lemma}

\begin{proof}
For sample \(z_j\), let \(\boldsymbol{g}_j=\nabla\ell(\boldsymbol{\theta}_0,z_j)\). Taylor expansion of one forward quotient gives
\begin{align*}
&\frac{\ell(\boldsymbol{\theta}_0+2\alpha\boldsymbol{u}_{jk},z_j)
-\ell(\boldsymbol{\theta}_0,z_j)}{2\alpha}\\
&\quad=\boldsymbol{g}_j^\top\boldsymbol{u}_{jk}
+\alpha\boldsymbol{u}_{jk}^\top\nabla^2\ell(\boldsymbol{\theta}_0,z_j)
\boldsymbol{u}_{jk}+\rho_{jk},
\end{align*}
where \(|\rho_{jk}|\leq(2L_H\alpha^2/3)\|\boldsymbol{u}_{jk}\|_2^3\). After multiplication by \(\boldsymbol{u}_{jk}\), the finite-radius error is at most
\[
\alpha M\|\boldsymbol{u}_{jk}\|_2^3
+\frac{2L_H\alpha^2}{3}\|\boldsymbol{u}_{jk}\|_2^4.
\]
For the random sketch, Gaussian moments give
\begin{align*}
\mathbb{E}[(\boldsymbol{u}\boldsymbol{u}^\top-\boldsymbol{I})\boldsymbol{g}]&=\boldsymbol{0},\\
\mathbb{E}\|(\boldsymbol{u}\boldsymbol{u}^\top-\boldsymbol{I})\boldsymbol{g}\|_2^2
&=(d+1)\|\boldsymbol{g}\|_2^2.
\end{align*}
Independence across the \(qK\) directions and Assumption~\ref{assusgradb} therefore imply
\[
\mathbb{E}\left\|\frac{1}{qK}\sum_{j=1}^{q}\sum_{k=1}^{K}
(\boldsymbol{u}_{jk}\boldsymbol{u}_{jk}^\top-\boldsymbol{I})\boldsymbol{g}_j
\right\|_2^2\leq\frac{(d+1)G^2}{qK}.
\]
Markov's inequality applied to the squared norm bounds this component by \(G\sqrt{(d+1)/(qK\delta)}\) with probability at least \(1-\delta\). Adding the realized forward-difference remainder proves Eq.~\eqref{eq:zo-error-supp}.
\end{proof}

\subsubsection{Steps 1--2: ZO Estimation Error (Term A)}

Apply Lemma~\ref{lemma:zo-gradient} to the proxy subset and the reference subset. A union bound gives, with probability at least \(1-\delta_0\),
\begin{align*}
\|\widetilde{\boldsymbol{g}}_{\mathrm{proxy}}
-\boldsymbol{g}_{\mathrm{proxy}}\|_2&\leq E_p,\\
\|\widetilde{\boldsymbol{g}}_0
-\nabla\mathcal{L}(\boldsymbol{\theta}_0)\|_2&\leq E_0.
\end{align*}
Define the signed estimation errors
\begin{align*}
\boldsymbol{e}_p
&=\widetilde{\boldsymbol{g}}_{\mathrm{proxy}}
-\boldsymbol{g}_{\mathrm{proxy}},\\
\boldsymbol{e}_0
&=\widetilde{\boldsymbol{g}}_0
-\nabla\mathcal{L}(\boldsymbol{\theta}_0),
\end{align*}
and the Step~2 residual
\begin{align*}
\boldsymbol{r}_{\mathrm{opt}}
&=\nabla\mathcal{L}(\boldsymbol{\theta}_0+\Delta)
-\widetilde{\boldsymbol{g}}_0-\widetilde{\boldsymbol{g}}_{\mathrm{proxy}},\\
\|\boldsymbol{r}_{\mathrm{opt}}\|_2&\leq\tau.
\end{align*}
The fundamental theorem of calculus gives the exact identity
\[
\nabla\mathcal{L}(\boldsymbol{\theta}_0+\Delta)
-\nabla\mathcal{L}(\boldsymbol{\theta}_0)
=\overline{\boldsymbol{H}}(\Delta)\Delta.
\]
Consequently,
\[
\overline{\boldsymbol{H}}(\Delta)\Delta
=\boldsymbol{g}_{\mathrm{proxy}}+\boldsymbol{e}_p+\boldsymbol{e}_0
+\boldsymbol{r}_{\mathrm{opt}}.
\]
Therefore, the direction error has the exact decomposition
\begin{align*}
\boldsymbol{v}_{\mathrm{FO}}-\Delta
&=(\boldsymbol{H}_0^{-1}
-\overline{\boldsymbol{H}}(\Delta)^{-1})
\boldsymbol{g}_{\mathrm{proxy}}\\
&\quad-\overline{\boldsymbol{H}}(\Delta)^{-1}
(\boldsymbol{e}_p+\boldsymbol{e}_0+\boldsymbol{r}_{\mathrm{opt}}).
\end{align*}
Combining this identity with the initial score bound yields the three error terms used in the proof:
\begin{align}
&|\mathcal{I}_{\mathrm{FO}}(z_i)
-\widehat{\mathcal{I}}_{\mathrm{ZOIF}}(z_i)|\notag\\
&\leq m_pG\underbrace{\left\|
\overline{\boldsymbol{H}}(\Delta)^{-1}
(\boldsymbol{e}_p+\boldsymbol{e}_0+\boldsymbol{r}_{\mathrm{opt}})
\right\|_2}_{\mathrm{Term~A}}\notag\\
&\quad+m_pG\underbrace{\left\|
(\boldsymbol{H}_0^{-1}-\overline{\boldsymbol{H}}(\Delta)^{-1})
\boldsymbol{g}_{\mathrm{proxy}}
\right\|_2}_{\mathrm{Term~B}}\notag\\
&\quad+m_p\underbrace{|\eta_i(\Delta)|}_{\mathrm{Term~C}}.
\label{eq:abc-decomposition-supp}
\end{align}
Term A captures the ZO gradient and finite-solver errors, Term B captures the local-curvature error, and Term C captures the Step~3 directional finite-difference error. We now bound each term separately.
Assumption~\ref{h_inv_bounded} yields
\begin{align}
\|\Delta\|_2&\leq\frac{G+E_\tau}{\mu},
\label{eq:delta-norm-supp}\\
\left\|\overline{\boldsymbol{H}}(\Delta)^{-1}
(\boldsymbol{e}_p+\boldsymbol{e}_0+\boldsymbol{r}_{\mathrm{opt}})\right\|_2
&\leq\frac{E_\tau}{\mu}.
\label{eq:term-a-supp}
\end{align}
Term A therefore includes the proxy and reference subsampling discrepancies, Gaussian sketching error, forward-difference bias, and the residual of the finite stochastic Step~2 solve. Exact local stationarity is the special case \(\tau=0\).

\subsubsection{Step 2: Local Curvature Error (Term B)}

Hessian Lipschitzness implies
\begin{align*}
\|\overline{\boldsymbol{H}}(\Delta)-\boldsymbol{H}_0\|_2
&\leq\int_0^1L_Ht\|\Delta\|_2\,dt\\
&=\frac{L_H}{2}\|\Delta\|_2.
\end{align*}
The resolvent identity and Assumption~\ref{h_inv_bounded} give
\begin{align*}
&\|\boldsymbol{H}_0^{-1}-\overline{\boldsymbol{H}}(\Delta)^{-1}\|_2\\
&\quad\leq\|\boldsymbol{H}_0^{-1}\|_2
\|\overline{\boldsymbol{H}}(\Delta)-\boldsymbol{H}_0\|_2
\|\overline{\boldsymbol{H}}(\Delta)^{-1}\|_2\\
&\quad\leq\frac{L_H}{2\mu^2}\|\Delta\|_2.
\end{align*}
Combining this result with Eq.~\eqref{eq:delta-norm-supp} gives
\begin{align}
&\|(\boldsymbol{H}_0^{-1}-\overline{\boldsymbol{H}}(\Delta)^{-1})
\boldsymbol{g}_{\mathrm{proxy}}\|_2\\
&\quad\leq\frac{L_HG(G+E_\tau)}{2\mu^3}.
\label{eq:term-b-supp}
\end{align}
This integrated-Hessian argument is exact and does not invoke a vector-valued mean-value theorem.

\subsubsection{Step 3: Directional Finite-Difference Error (Term C)}

Taylor's theorem with a Lipschitz-Hessian remainder gives
\begin{align*}
\ell(\boldsymbol{\theta}_0+\alpha\Delta,z_i)
&=\ell(\boldsymbol{\theta}_0,z_i)\\
&\quad+\alpha\nabla\ell(\boldsymbol{\theta}_0,z_i)^\top\Delta\\
&\quad+\frac{\alpha^2}{2}\Delta^\top
\nabla^2\ell(\boldsymbol{\theta}_0,z_i)\Delta\\
&\quad+R_i(\alpha),
\end{align*}
where \(|R_i(\alpha)|\leq(L_H\alpha^3/6)\|\Delta\|_2^3\). Hence
\begin{align*}
&|d_i(\Delta)-\nabla\ell(\boldsymbol{\theta}_0,z_i)^\top\Delta|\\
&\quad\leq\frac{M\alpha}{2}\|\Delta\|_2^2
+\frac{L_H\alpha^2}{6}\|\Delta\|_2^3\\
&\quad\leq\frac{M\alpha}{2}\left(\frac{G+E_\tau}{\mu}\right)^2\\
&\qquad+\frac{L_H\alpha^2}{6}
\left(\frac{G+E_\tau}{\mu}\right)^3.
\end{align*}
Substitution of the three bounds into Eq.~\eqref{eq:abc-decomposition-supp} proves Theorem~\ref{theorembound}.

\subsection{Finite Step~2 Runs}

Theorem~\ref{theorembound} is stated directly for a finite Step~2 output through \(\tau\). A local stationary point sets \(\tau=0\); a stopped stochastic run remains covered when its full-objective residual is bounded by \(\tau>0\). The result is therefore conditional on the quality of the returned local point and does not claim that every task, dataset, or parameter choice makes the ZO solver converge. If the norm projection is active, the analogous statement uses a projected gradient-mapping residual and must additionally account for the realized projection effect. Under analogous local effective-rank conditions, \citet[Theorem~1 and Lemma~3]{malladi2023mezo} show that the per-step and global convergence rates of ZO-SGD can depend on local effective rank rather than ambient dimension. Their setting differs from our auxiliary problem, so this result supports the plausibility of Step~2 without supplying its residual bound.

\subsection{Proof of Theorem~\ref{theoremboundLowrH}}

Write \(\Delta=\boldsymbol{U}_r\boldsymbol{\xi}\),
\(\boldsymbol{g}_{\mathrm{proxy},r}=\boldsymbol{U}_r^\top
\boldsymbol{g}_{\mathrm{proxy}}\), and
\(\overline{\boldsymbol{H}}_r(\boldsymbol{\xi})
=\boldsymbol{U}_r^\top\overline{\boldsymbol{H}}(\Delta)
\boldsymbol{U}_r\). Projecting the fundamental-theorem-of-calculus identity used above gives
\[
\overline{\boldsymbol{H}}_r(\boldsymbol{\xi})\boldsymbol{\xi}
=\boldsymbol{g}_{\mathrm{proxy},r}
+\boldsymbol{e}_{p,r}+\boldsymbol{e}_{0,r}
+\boldsymbol{r}_{\mathrm{opt},r}.
\]
Here \(\boldsymbol{e}_{p,r}=\boldsymbol{U}_r^\top\boldsymbol{e}_p\),
\(\boldsymbol{e}_{0,r}=\boldsymbol{U}_r^\top\boldsymbol{e}_0\), and
\(\boldsymbol{r}_{\mathrm{opt},r}=\boldsymbol{U}_r^\top
\boldsymbol{r}_{\mathrm{opt}}\).
The projected curvature condition bounds both
\(\|\boldsymbol{H}_{0,r}^{-1}\|_2\) and
\(\|\overline{\boldsymbol{H}}_r(\boldsymbol{\xi})^{-1}\|_2\)
by \(1/\mu\). Thus, the decomposition for Terms A and B applies verbatim in the projected coordinates. The full-space subsampling bounds remain valid because multiplication by \(\boldsymbol{U}_r^\top\) is nonexpansive.

For the ZO terms, let \(\boldsymbol{q}_{jk}\sim\mathcal{N}(\boldsymbol{0},\boldsymbol{I}_r)\) and \(\boldsymbol{u}_{jk}=\boldsymbol{U}_r\boldsymbol{q}_{jk}\). Orthogonality gives \(\|\boldsymbol{u}_{jk}\|_2=\|\boldsymbol{q}_{jk}\|_2\), and the projected Gaussian-sketch identity becomes
\[
\mathbb{E}\|(\boldsymbol{q}\boldsymbol{q}^\top-\boldsymbol{I}_r)
\boldsymbol{g}_r\|_2^2=(r+1)\|\boldsymbol{g}_r\|_2^2.
\]
Hence Eq.~\eqref{eq:zo-error-supp} holds with \(d\) replaced by \(r\). Finally, \(\|\Delta\|_2=\|\boldsymbol{\xi}\|_2\), so the Step~3 Taylor bound is unchanged. These three observations yield Theorem~\ref{theoremboundLowrH} with \(E_\tau\) replaced by \(E_{r,\tau}\). The target direction in the original coordinates is
\[
\boldsymbol{H}_0^{\dagger}\boldsymbol{g}_{\mathrm{proxy}}
=\boldsymbol{U}_r\boldsymbol{H}_{0,r}^{-1}
\boldsymbol{U}_r^\top\boldsymbol{g}_{\mathrm{proxy}}.
\]
Comparison with the unrestricted full-space \texttt{FOIF} requires adding the norm of the omitted projection component; low effective rank alone does not make this bias zero.

\subsection{Unified Definitions of the Four Step~2 Variants}
Let \(\widetilde{\boldsymbol{g}}_{q,b}\) denote the ZO gradient estimate for reference-pool sample \(q\) in parameter block \(b\), and let \(N\) be the size of the Step~2 subset. All sample gradients below use the forward estimator analyzed above. The four reported variants are:
\begin{itemize}
    \item \texttt{ZOIF\_opt}: for a minibatch loss \(\mathcal{L}_{B}\), the stochastic auxiliary objective is
    \begin{align*}
    J_B(\boldsymbol{\delta})
    &=\mathcal{L}_{B}(\boldsymbol{\theta}_0+\boldsymbol{\delta})
    -\mathcal{L}_{B}(\boldsymbol{\theta}_0)\\
    &\quad-\langle\widetilde{\boldsymbol{g}}_{0,B},\boldsymbol{\delta}\rangle
    -\langle\widetilde{\boldsymbol{g}}_{\mathrm{proxy}},\boldsymbol{\delta}\rangle.
    \end{align*}
    Both \(\widetilde{\boldsymbol{g}}_{0,B}\) and the update direction for \(J_B\) are estimated by \texttt{ZOGE}, followed by projection onto \(\|\boldsymbol{\delta}\|_2\leq0.05\). Theorem~\ref{theorembound} gives the local bound when the projection is inactive; an active projection requires the projected residual described above.
    \item \texttt{ZOIF\_hvp}: this variant approximates the curvature-aware direction using a DataInf-style damped block-wise empirical-Fisher inverse. For block \(b\), set
    \[
    \lambda_b=\frac{1}{10N}\sum_{q=1}^{N}\operatorname{mean}
    (\widetilde{\boldsymbol{g}}_{q,b}^{\odot2}).
    \]
    The closed-form operator is
    \[
    \widehat{\boldsymbol{B}}_{\mathrm{hvp},b}
    =\frac{1}{N\lambda_b}\sum_{q=1}^{N}\left(
    \boldsymbol{I}-\frac{\widetilde{\boldsymbol{g}}_{q,b}\widetilde{\boldsymbol{g}}_{q,b}^{\top}}
    {\lambda_b+\|\widetilde{\boldsymbol{g}}_{q,b}\|_2^2}\right).
    \]
    For \(\lambda_b>0\), \(\|\widehat{\boldsymbol{B}}_{\mathrm{hvp},b}\|_2\leq1/\lambda_b\).
    \item \texttt{ZOIF\_dia}: the implementation forms the coordinate-wise empirical second moment
    \[
    \widehat{\boldsymbol{h}}=\frac{1}{N}\sum_{q=1}^{N}\widetilde{\boldsymbol{g}}_q^{\odot2}
    \]
    and uses \(\widehat{\boldsymbol{B}}_{\mathrm{dia}}=\operatorname{Diag}(\widehat{\boldsymbol{h}}+10^{-8}\boldsymbol{1})^{-1}\). Hence \(\|\widehat{\boldsymbol{B}}_{\mathrm{dia}}\|_2\leq10^8\). This construction adopts the diagonal-approximation principle used by methods such as SOUL~\citep{zhang2024soul}. In the forward-only setting, the empirical second moment of the ZO gradient supplies the diagonal proxy.
    \item \texttt{ZOIF\_ide}: \(\widehat{\boldsymbol{B}}_{\mathrm{ide}}=\boldsymbol{I}\), so \(\widehat{\boldsymbol{v}}_{\mathrm{ide}}=\widetilde{\boldsymbol{g}}_{\mathrm{proxy}}\).
\end{itemize}
\subsection{Theoretical Extensions to Plug-in Step~2 Operators}

The main theorem analyzes Algorithm~1. For a plug-in Step~2 variant \(a\in\{\mathrm{hvp,dia,ide}\}\), suppose \(\boldsymbol{H}_{0,\lambda}=\boldsymbol{H}_0+\lambda\boldsymbol{I}\) is invertible, and let \(\boldsymbol{v}_{\lambda}=\boldsymbol{H}_{0,\lambda}^{-1}\boldsymbol{g}_{\mathrm{proxy}}\), \(\mathcal{I}_{\lambda}(z_i)=-m_p\nabla\ell(\boldsymbol{\theta}_0,z_i)^\top\boldsymbol{v}_{\lambda}\), and \(\widehat{\boldsymbol{v}}_a=\widehat{\boldsymbol{B}}_a\widetilde{\boldsymbol{g}}_{\mathrm{proxy}}\). Define
\[
\rho_a=\|(\boldsymbol{H}_{0,\lambda}^{-1}-\widehat{\boldsymbol{B}}_a)
\boldsymbol{g}_{\mathrm{proxy}}\|_2.
\]
Then
\begin{align}
\|\boldsymbol{v}_{\lambda}-\widehat{\boldsymbol{v}}_a\|_2
&\leq\rho_a+\|\widehat{\boldsymbol{B}}_a\|_2\notag\\
&\quad\cdot
\|\widetilde{\boldsymbol{g}}_{\mathrm{proxy}}-\boldsymbol{g}_{\mathrm{proxy}}\|_2.
\label{eq:plugin-operator-bound}
\end{align}

\begin{proposition}\label{prop:plugin-extension-supp}
If the Step~3 segment satisfies Assumptions~\ref{hessian_bounded}, \ref{assusgradb}, and~\ref{hessian_lip}, define
\[
\widehat{\mathcal{I}}_a(z_i)
=-m_p\frac{\ell(\boldsymbol{\theta}_0+\alpha\widehat{\boldsymbol{v}}_a,z_i)
-\ell(\boldsymbol{\theta}_0,z_i)}{\alpha}.
\]
Then
\begin{align*}
|\mathcal{I}_{\lambda}(z_i)-\widehat{\mathcal{I}}_a(z_i)|
&\leq m_pG\|\boldsymbol{v}_{\lambda}-\widehat{\boldsymbol{v}}_a\|_2\\
&\quad+\frac{m_pM\alpha}{2}\|\widehat{\boldsymbol{v}}_a\|_2^2\\
&\quad+\frac{m_pL_H\alpha^2}{6}\|\widehat{\boldsymbol{v}}_a\|_2^3.
\end{align*}
\end{proposition}
\begin{proof}
Adding and subtracting
\(\widehat{\boldsymbol{B}}_a\boldsymbol{g}_{\mathrm{proxy}}\)
gives
\begin{align*}
\boldsymbol{v}_{\lambda}-\widehat{\boldsymbol{v}}_a
&=(\boldsymbol{H}_{0,\lambda}^{-1}
-\widehat{\boldsymbol{B}}_a)\boldsymbol{g}_{\mathrm{proxy}}\\
&\quad+\widehat{\boldsymbol{B}}_a
(\boldsymbol{g}_{\mathrm{proxy}}
-\widetilde{\boldsymbol{g}}_{\mathrm{proxy}}),
\end{align*}
which proves Eq.~\eqref{eq:plugin-operator-bound}. Next, add and subtract
\(-m_p\nabla\ell(\boldsymbol{\theta}_0,z_i)^\top
\widehat{\boldsymbol{v}}_a\) between
\(\mathcal{I}_{\lambda}(z_i)\) and
\(\widehat{\mathcal{I}}_a(z_i)\). Cauchy--Schwarz bounds the resulting direction-error term by
\(m_pG\|\boldsymbol{v}_{\lambda}-\widehat{\boldsymbol{v}}_a\|_2\).
The same Lipschitz-Hessian Taylor expansion used for Term C bounds the remaining forward-difference term by the final two terms in the proposition.
\end{proof}
Equation~\eqref{eq:plugin-operator-bound} separates ZO gradient-estimation error from each variant's structural operator mismatch; this extension is separate from the Algorithm~1 theorem.
\section{Supplementary Experimental Details}
\subsection{Datasets and Evaluation}\label{testdetail}

Our evaluation code framework follows that of~\citet{ivison2025large}. 

MMLU~\citep{hendryckstest2021}: We use the official validation split as proxy data and the official test split for evaluation. Model performance is measured using Average Accuracy. During evaluation, we use a 5-shot setting, consistent with~\citet{ivison2025large}.

BBH (Big-Bench Hard)~\citep{suzgun2022challenging}: This benchmark comprises 27 challenging reasoning tasks across diverse domains. For each task, we randomly sample 1/10 of the data as proxy data for influence computation, with the remaining 9/10 used as test data. Evaluation is based on Average Accuracy. Evaluation follows a 3-shot setting, as in~\citet{ivison2025large}.

SQuAD~\citep{rajpurkar-etal-2016-squad}: A reading comprehension dataset with official training and validation splits. We sample 1,174 examples from the training split as proxy data and use the official validation split for evaluation. Following~\citet{ivison2025large}, model performance is assessed using the F1 score, and no in-context examples are used during evaluation on SQuAD.

All evaluation metrics were computed using the corresponding functions from the Hugging Face evaluate library.

All datasets are converted into a common evaluation format. Their proxy-to-test sizes are approximately 1:9, and the same fixed splits are used for all compared methods to ensure a fair comparison.
The number of proxy samples used for IFs computation and the number of test samples used for final performance evaluation are summarized in Table~\ref{tab:dataset_sizes}.

\begin{table}[htbp]
\centering
\begin{tabular}{lcc}
\toprule
Dataset & $|\mathcal{D}_{\text{proxy}}|$ & $|\mathcal{D}_{\text{test}}|$ \\
\midrule
MMLU   & 1,531 & 14,042 \\
BBH    & 649   & 5,862  \\
SQuAD  & 1,174 & 10,570 \\
\bottomrule
\end{tabular}
\caption{Sample sizes for the proxy and test splits of each dataset.}
\label{tab:dataset_sizes}
\end{table}

\subsection{Implementation and Hyperparameter Settings}\label{Paramset}
Influence functions are computed in bfloat16 for all models. On \texttt{Llama} and \texttt{Qwen}, we use LoRA (rank 16, scaling 64) for IFs computation and switch to a higher-capacity LoRA (rank 64, scaling 128) during fine-tuning. For \texttt{pythia-70m}, we perform full-parameter gradient computation and fine-tuning.

In Step 1, since the proxy dataset sizes for all three datasets are relatively small, we set \( |\mathcal{S}|=m_p\) to ensure estimation accuracy.

When computing IFs, \texttt{ZOIF\_hvp} follows the block-wise empirical-Fisher inverse approximation of~\citet{Kwonetal2023}. To reduce computation, we draw a random subset of 2,000 examples from the data pool for Step~2. The implementation sets each block damping value to one tenth of the mean coordinate-wise ZO gradient second moment over this subset. \texttt{ZOIF\_dia} uses the same ZO gradients to form a coordinate-wise empirical second moment with a $10^{-8}$ stabilizer. For a fair comparison, \texttt{ZOIF\_dia}, \texttt{ZOIF\_opt}, and their matched \texttt{FOIF} counterparts use the same subset.

When computing Step 2 with our \texttt{ZOIF\_opt} method, we employ a ZO optimizer~\citep{malladi2023mezo} to save memory.
Since we approximate the IHVP using $\Delta_{\boldsymbol{\theta}}$, and parameter updates remain within a small range, we use a small batch size of 8 and 50 optimization steps. To prevent gradient explosion, we apply gradient clipping with a threshold of 0.05. For \texttt{pythia-70m}, we set the learning rate to $1\text{e-}6$. For all other models, which optimize only the LoRA layers, the learning rate is set to $1\text{e-}4$.
In contrast, \texttt{FOIF\_opt} uses exact first-order gradients for gradient descent, with all other hyperparameters kept identical to those of \texttt{ZOIF\_opt}.

For memory and runtime comparison, we use the MMLU proxy set in Step 1, a subset of 2,000 examples drawn from the data pool in Step 2, and the full data pool in Step 3. The \texttt{pythia-70m} model is relatively small and sensitive to fluctuations caused by data tokenization. To mitigate this variability in the memory comparison for Step 1 and 2, we construct a synthetic dataset by duplicating a single training example.

For fine-tuning, we use the AdamW optimizer~\citep{loshchilov2019de} with a learning rate of $2\text{e-}4$ and fine-tune the \texttt{Llama} and \texttt{Qwen} series models for three epochs. Given that the \texttt{pythia-70m} model is significantly smaller, we use a smaller learning rate of $1\text{e-}6$ and train it for one epoch.
Due to GPU memory constraints, the per-device batch size is set to 1, combined with gradient accumulation over 8 steps to achieve an effective batch size of 8. FlashAttention-2 is enabled when fine-tuning the \texttt{Llama} and \texttt{Qwen} series models and disabled for \texttt{pythia-70m}; model checkpoint saving is disabled to reduce storage overhead.

\subsection{Details of Baseline Methods}

\begin{itemize}
    \item \texttt{FOIF}: The standard first-order influence function method. For fair comparison, we implement first-order counterparts corresponding exactly to our ZO variants, namely \texttt{FOIF\_opt}, \texttt{FOIF\_hvp}, \texttt{FOIF\_dia}, and \texttt{FOIF\_ide}.
    \item \texttt{LESS} \cite{xia2024less} and \texttt{TRAK} \cite{park2023trak}: Gradient-based influence estimation methods that use gradient projection and random sketching for scalable data attribution. Both require BP to compute the gradient features used by their scores.
    \item \texttt{EMB}: A data selection method based on embedding-space similarity~\citep{ni2022large,lee2024nv}. We use the \texttt{embedding-gemma-300M} model from~\citet{embedding_gemma_2025} to compute embeddings for each sample, then calculate the cosine similarity between each candidate sample and the average embedding of the proxy set. Samples are ranked from highest to lowest similarity, and the top 10000 are selected.
    
    \item \texttt{PPL}: A perplexity-based data selection strategy~\citep{yin2024compute}. We compute the perplexity of each candidate sample using the target LLM and sort the samples in descending order of perplexity. Following~\citet{yin2024compute}, we consider two sampling schemes: \texttt{Top\_PPL}, which selects the top 10,000 samples with the highest perplexity, and \texttt{Mid\_PPL}, which selects 10,000 samples from the middle of the perplexity distribution.
    
    \item \texttt{Random}: Randomly samples 10000 examples from the candidate pool for fine-tuning.
    
    \item \texttt{Orig}: The performance of the original, un-fine-tuned model, serving as a lower-bound reference.
\end{itemize}

\subsection{Hardware and Reproducibility}
All reported experimental results are averaged over three independent runs.

Primary model fine-tuning and downstream evaluations were conducted on NVIDIA A100 GPUs with 80GB memory, using two Hygon C86 7390 32-core processors, Ubuntu 22.04, Python 3.12, PyTorch 2.7.0, and CUDA 11.8. For rigorous efficiency cross-validation, we additionally used a domestic heterogeneous AI accelerator with 64GB device memory. This environment used one Hygon C86 7185 32-core processor (6 CPU cores allocated), 115GB system memory, Ubuntu 22.04, Python 3.11, PyTorch 2.9.0, and DTK 26.04.
While absolute memory and runtime metrics varied slightly across hardware platforms ($\pm 0.02$ GB memory, $\sim 1.5\times$ speedup on A100), the relative efficiency advantages of \texttt{ZOIF} over \texttt{FOIF} remained consistent, confirming hardware-agnostic scalability.

\section{Supplementary Experimental Results}\label{aresults}
\setlength{\tabcolsep}{2pt}

\subsection{Comparison of Memory and Time Efficiency}
\subsubsection{Memory Efficiency}\label{memory2}
The peak GPU memory results for all three steps are shown in Table~\ref{tab:llm_memory_all_steps}, demonstrating similar findings.
In particular, \texttt{ZOIF\_opt} exhibits a clear advantage over \texttt{FOIF\_opt} in terms of GPU memory usage.
Gradient-based methods such as \texttt{LESS} and \texttt{TRAK} require backpropagation and storage of per-sample gradient features. On larger models, this memory overhead exceeded our hardware budget; consequently, we could not run these methods on \texttt{Llama-2-7B} or \texttt{Qwen3-1.7B}. This limitation further motivates forward-only influence estimation for memory-constrained settings.

\begin{table*}[t]
\centering
{\small
\setlength{\tabcolsep}{2pt}
\begin{tabular}{@{}l *{10}{c} l@{}}
\toprule
\multirow{4}{*}{LLM} & \multicolumn{5}{c}{\texttt{FOIF}} & \multicolumn{5}{c}{\texttt{ZOIF}} & \multirow{4}{*}{Notes} \\
\cmidrule(lr){2-6} \cmidrule(lr){7-11}
& \multicolumn{1}{c}{Step 1} & \multicolumn{3}{c}{Step 2} & Step 3 & \multicolumn{1}{c}{Step 1} & \multicolumn{3}{c}{Step 2} & Step 3 & \\
\cmidrule(lr){2-2} \cmidrule(lr){3-5} \cmidrule(lr){6-6} \cmidrule(lr){7-7} \cmidrule(lr){8-10} \cmidrule(lr){11-11}
& & \texttt{FOIF\_dia} & \texttt{FOIF\_hvp} & \texttt{FOIF\_opt} & & & \texttt{ZOIF\_dia} & \texttt{ZOIF\_hvp} & \texttt{ZOIF\_opt} & & \\
\midrule
\texttt{pythia-70m}     & 0.68 & 0.94 & 0.86 & 2.70 & 0.92 & 0.63 & 0.90 & 0.63 & 1.72 & 0.51 & FULL \\
\texttt{Llama-3.2-1B}      & 3.91 & 3.94 & 3.96 & 10.67 & 3.91 & 2.96 & 2.97 & 2.97 & 3.00 & 2.95 & LORA \\
\texttt{Llama-2-7B}       & 15.74 & 15.56 & 15.62 & 34.39 & 15.46 & 13.08 & 13.15 & 13.15 & 13.27 & 13.02 & LORA \\
\texttt{Qwen3-0.6B}       & 3.09 & 3.11 & 3.13 & 15.77 & 3.10 & 1.92 & 1.94 & 1.94 & 1.98 & 1.91 & LORA \\
\texttt{Qwen3-1.7B}       & 5.68 & 5.73 & 5.74 & 21.46 & 5.69 & 4.03 & 4.06 & 4.06 & 4.11 & 4.01 & LORA \\
\bottomrule
\end{tabular}
}
\caption{Peak GPU memory usage (in GB) for \texttt{FOIF} and \texttt{ZOIF} across all steps.}
\label{tab:llm_memory_all_steps}
\end{table*}

\subsubsection{Time Efficiency}\label{time2}

Table~\ref{tab:llm_time_all_steps} presents a step-by-step breakdown of the runtime. Although the zeroth-order (ZO) perturbation leads to increased computational overhead in the earlier steps for \texttt{ZOIF}, its overall runtime remains on par with that of \texttt{FOIF}.

\begin{table*}[t]
\centering
{\small
\setlength{\tabcolsep}{2pt}
\begin{tabular}{@{}l *{10}{c} l@{}}
\toprule
\multirow{4}{*}{LLM} & \multicolumn{5}{c}{\texttt{FOIF}} & \multicolumn{5}{c}{\texttt{ZOIF}} & \multirow{4}{*}{Notes} \\
\cmidrule(lr){2-6} \cmidrule(lr){7-11}
& \multicolumn{1}{c}{Step 1} & \multicolumn{3}{c}{Step 2} & Step 3 & \multicolumn{1}{c}{Step 1} & \multicolumn{3}{c}{Step 2} & Step 3 & \\
\cmidrule(lr){2-2} \cmidrule(lr){3-5} \cmidrule(lr){6-6} \cmidrule(lr){7-7} \cmidrule(lr){8-10} \cmidrule(lr){11-11}
& & \texttt{FOIF\_dia} & \texttt{FOIF\_hvp} & \texttt{FOIF\_opt} & & & \texttt{ZOIF\_dia} & \texttt{ZOIF\_hvp} & \texttt{ZOIF\_opt} & &\\
\midrule
\texttt{pythia-70m}     & 0.017 & 0.021 & 0.040 & 0.033 & 37.60 & 0.27 & 0.33 & 0.58 & 0.60 & 22.50 & FULL \\
\texttt{Llama-3.2-1B}      & 0.058 & 0.083 & 0.175 & 0.050 & 34.00 & 0.80 & 0.92 & 2.28 & 0.80 & 30.00 & LORA \\
\texttt{Llama-2-7B}       & 0.100 & 0.150 & 0.320 & 0.050 & 43.00 & 2.00 & 2.50 & 5.00 & 2.00 & 40.00 & LORA \\
\texttt{Qwen3-0.6B}       & 0.100 & 0.150 & 0.310 & 0.050 & 40.00 & 1.50 & 1.70 & 3.50 & 1.00 & 31.75 & LORA \\
\texttt{Qwen3-1.7B}       & 0.100 & 0.150 & 0.320 & 0.050 & 40.30 & 1.50 & 1.80 & 3.50 & 1.60 & 32.00 & LORA \\
\bottomrule
\end{tabular}
}
\caption{Computation time (in hours) for \texttt{FOIF} and \texttt{ZOIF} across all steps.}
\label{tab:llm_time_all_steps}
\end{table*}

\subsubsection{Comparison with Memory Optimization Techniques}
\label{app:memory-optimization-comparison}
First-order methods such as \texttt{FOIF} can reduce activation memory using Gradient Checkpointing (GC) and Activation Offloading. We evaluate these techniques on \texttt{Llama-3.2-1B} ($K=30$). As shown in Table~\ref{tab:memory_optimization}, applying both techniques reduces the peak memory of \texttt{FOIF} from 3.91 GB to 3.18 GB, which remains \textbf{7.4\% higher} than the 2.96 GB required by \texttt{ZOIF}. The optimizations also increase the runtime of \texttt{FOIF} relative to its unoptimized setting: GC, offloading, and their combination require 0.125h, 0.750h, and 0.167h, respectively, compared with 0.058h for the unoptimized method.

\begin{table}[t]
\centering
{\small
\setlength{\tabcolsep}{4pt}
\begin{tabular}{@{}lcc@{}}
\toprule
\textbf{Method} & \textbf{Memory (GB)} & \textbf{Time (h)} \\
\midrule
\texttt{FOIF} & 3.91 & 0.058 \\
\textbf{\texttt{ZOIF}} & \textbf{2.96} & 0.800 \\
\texttt{FOIF+GC} & 3.21 & 0.125 \\
\texttt{FOIF+Offload} & 3.19 & 0.750 \\
\texttt{FOIF+Both} & 3.18 & 0.167 \\
\bottomrule
\end{tabular}
}
\caption{Step 1 memory and runtime on \texttt{Llama-3.2-1B} ($K=30$).}
\label{tab:memory_optimization}
\end{table}

These Step~1 results show that \texttt{ZOIF} retains a lower measured memory footprint than the evaluated first-order configurations, including those using gradient checkpointing and activation offloading.

\subsection{Comparison of Fine-tuning Performance Across Different Data Selection Methods}\label{strategy2} 

\subsubsection{Results on \texttt{pythia-70m}}\label{py2}

Results for \texttt{pythia-70m} fine-tuned using data selected under Strategy~1 are shown in Table~\ref{tab:pythia-70m_results}.
Limited by the inherent capacity of \texttt{pythia-70m}, fine-tuning yields only marginal improvements in model performance.
Nevertheless, we observe that \texttt{ZOIF\_hvp} achieves the highest average score and slightly outperforms the other methods. Moreover, the \texttt{ZOIF} variants as a whole demonstrate performance comparable to that of the \texttt{FOIF} approaches. The Strategy~2 results are shown in Table~\ref{tab:pythia-70m_alt} and exhibit a similar pattern.

\begin{table}[htbp]
\centering
\begin{tabular}{lcccc}
\toprule
\textbf{Method} & \textbf{\shortstack{MMLU\\Acc.}} & \textbf{\shortstack{BBH\\Acc.}} & \textbf{\shortstack{SQuAD\\F1}} & \textbf{\shortstack{Macro\\Avg.}} \\
\midrule
\texttt{ZOIF\_dia} & 23.45 & 7.24 & 0.64 & 10.44 $\pm$ 0.01 \\
\texttt{ZOIF\_hvp} & 23.17 & \best{9.25} & 0.65 & \best{11.03} $\pm$ 0.00 \\
\texttt{ZOIF\_ide} & 23.38 & 7.38 & 0.66 & 10.47 $\pm$ 0.01 \\
\texttt{ZOIF\_opt} & 23.30 & 6.31 & 0.67 & 10.09 $\pm$ 0.01 \\
\texttt{FOIF\_dia} & 23.29 & 7.03 & 0.59 & 10.30 $\pm$ 0.01 \\
\texttt{FOIF\_hvp} & \best{23.60} & 7.09 & 0.54 & 10.41 $\pm$ 0.02 \\
\texttt{FOIF\_ide} & 23.26 & 7.39 & 0.67 & 10.44 $\pm$ 0.00 \\
\texttt{FOIF\_opt} & 23.36 & 7.63 & \best{0.68} & 10.55 $\pm$ 0.01 \\
\texttt{EMB} & 23.09 & 6.08 & 0.52 & 9.90 $\pm$ 0.01 \\
\texttt{Top\_PPL} & 23.16 & 6.57 & 0.65 & 10.13 $\pm$ 0.01 \\
\texttt{Mid\_PPL} & 23.18 & 6.86 & 0.65 & 10.23 $\pm$ 0.01 \\
\texttt{Random} & 23.17 & 7.58 & 0.61 & 10.46 $\pm$ 0.00 \\
\texttt{Orig} & 22.95 & 1.44 & 0.58 & 8.32 \\
\bottomrule

\end{tabular}
\caption{Held-out test performance on \texttt{pythia-70m} (Strategy 1). The macro-average entry is reported as mean $\pm$ standard deviation over three independent runs.}
\label{tab:pythia-70m_results}
\end{table}

\begin{table}[htbp]
\centering
\begin{tabular}{lcccc}
\toprule
\textbf{Method} & \textbf{\shortstack{MMLU\\Acc.}} & \textbf{\shortstack{BBH\\Acc.}} & \textbf{\shortstack{SQuAD\\F1}} & \textbf{\shortstack{Macro\\Avg.}} \\
\midrule
\texttt{ZOIF\_dia} & 23.29 & 7.09 & \best{0.70} & 10.36 $\pm$ 0.02 \\
\texttt{ZOIF\_hvp} & 23.29 & 7.45 & 0.64 & 10.46 $\pm$ 0.01 \\
\texttt{ZOIF\_ide} & 23.35 & 8.59 & 0.55 & 10.83 $\pm$ 0.01 \\
\texttt{ZOIF\_opt} & 23.55 & 6.51 & 0.57 & 10.21 $\pm$ 0.02 \\
\texttt{FOIF\_dia} & 23.48 & 6.28 & 0.61 & 10.12 $\pm$ 0.01 \\
\texttt{FOIF\_hvp} & 23.19 & 7.14 & 0.62 & 10.32 $\pm$ 0.03 \\
\texttt{FOIF\_ide} & 23.33 & 7.57 & 0.63 & 10.51 $\pm$ 0.01 \\
\texttt{FOIF\_opt} & \best{23.63} & \best{9.51} & 0.56 & \best{11.23} $\pm$ 0.00 \\
\texttt{EMB} & 23.36 & 6.73 & 0.63 & 10.24 $\pm$ 0.01 \\
\texttt{Top\_PPL} & 23.17 & 6.57 & 0.65 & 10.13 $\pm$ 0.01 \\
\texttt{Mid\_PPL} & 23.18 & 6.87 & 0.65 & 10.23 $\pm$ 0.01 \\
\texttt{Random} & 23.17 & 7.58 & 0.61 & 10.46 $\pm$ 0.00 \\
\texttt{Orig} & 22.95 & 1.44 & 0.58 & 8.32 \\
\bottomrule
\end{tabular}
\caption{Held-out test performance on \texttt{pythia-70m} (Strategy 2). The macro-average entry is reported as mean $\pm$ standard deviation over three independent runs.}
\label{tab:pythia-70m_alt}
\end{table}

\subsubsection{Results on \texttt{Llama}}\label{lla2}

\paragraph{\texttt{Llama-3.2-1B} (Strategy~1).}
Evaluation results for \texttt{Llama-3.2-1B} under Strategy~1 are shown in Table~\ref{tab:Llama-3.2-1B_results}.  
\texttt{TRAK} achieves the highest average score (45.73), followed by \texttt{LESS} (44.53), with \texttt{ZOIF\_hvp} attaining 43.98---within 4\% of the best result. Crucially, both \texttt{TRAK} and \texttt{LESS} require backpropagation through the full computational graph, incurring activation memory that scales with model depth and sequence length. By contrast, \texttt{ZOIF} operates entirely via forward passes, achieving competitive downstream performance at a fraction of the memory cost. This illustrates the central advantage of \texttt{ZOIF}: it offers a practical, memory-efficient alternative for influence-based data selection, delivering performance close to gradient-based methods without the prohibitive memory demands. Furthermore, \texttt{ZOIF} consistently outperforms non-gradient baselines (\texttt{EMB}, \texttt{PPL}, \texttt{Random}) by a clear margin.

\begin{table}[htbp]
\centering
{\small
\setlength{\tabcolsep}{3pt}
\begin{tabular}{@{}lcccc@{}}
\toprule
\textbf{Method} & \textbf{\shortstack{MMLU\\Acc.}} & \textbf{\shortstack{BBH\\Acc.}} & \textbf{\shortstack{SQuAD\\F1}} & \textbf{\shortstack{Macro\\Avg.}} \\
\midrule
  \texttt{ZOIF\_dia} & 34.36 & 30.81 & 63.21 & 42.79 $\pm$ 0.92 \\
  \texttt{ZOIF\_hvp} & 34.23 & 29.63 & 68.09 & 43.98 $\pm$ 0.65 \\
  \texttt{ZOIF\_ide} & 33.85 & 30.57 & 64.27 & 42.90 $\pm$ 0.71 \\
  \texttt{ZOIF\_opt} & 31.37 & 30.77 & 65.80 & 42.64 $\pm$ 1.59 \\
  \texttt{FOIF\_dia} & 30.59 & 29.92 & 63.31 & 41.27 $\pm$ 1.23 \\
  \texttt{FOIF\_hvp} & 31.11 & \best{30.93} & 59.50 & 40.51 $\pm$ 1.38 \\
  \texttt{FOIF\_ide} & 32.43 & 30.55 & 60.85 & 41.28 $\pm$ 1.59 \\
  \texttt{FOIF\_opt} & 30.77 & 30.20 & 60.03 & 40.34 $\pm$ 1.32 \\
  \texttt{LESS} & 32.98 & 28.99 & 71.62 & 44.53 $\pm$ 0.92 \\
  \texttt{TRAK} & \best{35.65} & 29.38 & \best{72.17} & \best{45.73} $\pm$ 1.07 \\
  \texttt{EMB} & 34.49 & 30.70 & 59.21 & 41.47 $\pm$ 0.31 \\
  \texttt{Top\_PPL} & 31.98 & 28.05 & 57.39 & 39.14 $\pm$ 0.73 \\
  \texttt{Mid\_PPL} & 32.47 & 28.27 & 57.78 & 39.51 $\pm$ 0.60 \\
  \texttt{Random} & 32.79 & 30.58 & 62.31 & 41.89 $\pm$ 0.52 \\
  \texttt{Orig} & 30.69 & 28.92 & 47.42 & 35.67 \\
\bottomrule
\end{tabular}
}
\caption{Held-out test performance on \texttt{Llama-3.2-1B} (Strategy 1). The average is the unweighted macro-average of MMLU accuracy, BBH accuracy, and SQuAD F1. The macro-average entry is reported as mean $\pm$ standard deviation over three independent runs.}
\label{tab:Llama-3.2-1B_results}
\end{table}

Results for \texttt{Llama-3.2-1B} fine-tuned using data selected under Strategy~2 are shown in Table~\ref{tab:Llama-3.2-1B_alt}. Under this strategy, \texttt{TRAK} again achieves the highest average score, while \texttt{ZOIF} variants deliver competitive performance consistent with the main-text observations, reinforcing the memory--performance trade-off. For \texttt{Llama-2-7B}, results under Strategy~1 and Strategy~2 are reported in Tables~\ref{tab:llama-2-7B_results} and~\ref{tab:llama-2-7B_alt}, respectively, showing that \texttt{ZOIF} remains on par with \texttt{FOIF}.

\begin{table}[htbp]
\centering
\begin{tabular}{lcccc}
\toprule
\textbf{Method} & \textbf{\shortstack{MMLU\\Acc.}} & \textbf{\shortstack{BBH\\Acc.}} & \textbf{\shortstack{SQuAD\\F1}} & \textbf{\shortstack{Macro\\Avg.}} \\
\midrule
  \texttt{ZOIF\_dia} & 33.05 & 30.31 & 66.40 & 43.25 $\pm$ 1.04 \\
  \texttt{ZOIF\_hvp} & 33.17 & 31.15 & 58.45 & 40.92 $\pm$ 2.08 \\
  \texttt{ZOIF\_ide} & 33.05 & 30.34 & 63.17 & 42.19 $\pm$ 1.13 \\
  \texttt{ZOIF\_opt} & 31.33 & \best{31.26} & 65.08 & 42.55 $\pm$ 0.89 \\
  \texttt{FOIF\_dia} & 33.69 & 30.86 & 65.61 & 43.39 $\pm$ 1.73 \\
  \texttt{FOIF\_hvp} & 32.14 & 30.44 & 64.11 & 42.23 $\pm$ 1.19 \\
  \texttt{FOIF\_ide} & 31.51 & 30.86 & 68.49 & 43.62 $\pm$ 1.48 \\
  \texttt{FOIF\_opt} & 33.99 & 30.85 & 56.24 & 40.36 $\pm$ 0.66 \\
  \texttt{LESS} & 33.04 & 30.87 & 64.35 & 42.75 $\pm$ 1.21 \\
  \texttt{TRAK} & \best{36.26} & 30.69 & \best{73.92} & \best{46.96} $\pm$ 1.00 \\
  \texttt{EMB} & 34.70 & 30.58 & 69.44 & 44.91 $\pm$ 0.63 \\
  \texttt{Top\_PPL} & 32.96 & 28.73 & 59.18 & 40.29 $\pm$ 1.01 \\
  \texttt{Mid\_PPL} & 31.82 & 28.58 & 57.48 & 39.29 $\pm$ 0.66 \\
  \texttt{Random} & 32.23 & 30.71 & 62.11 & 41.68 $\pm$ 0.52 \\
  \texttt{Orig} & 30.69 & 28.92 & 47.42 & 35.67 \\
\bottomrule
\end{tabular}
\caption{Held-out test performance on \texttt{Llama-3.2-1B} (Strategy 2). The macro-average entry is reported as mean $\pm$ standard deviation over three independent runs.}
\label{tab:Llama-3.2-1B_alt}
\end{table}

\begin{table}[htbp]
\centering
\begin{tabular}{lcccc}
\toprule
\textbf{Method} & \textbf{\shortstack{MMLU\\Acc.}} & \textbf{\shortstack{BBH\\Acc.}} & \textbf{\shortstack{SQuAD\\F1}} & \textbf{\shortstack{Macro\\Avg.}} \\
\midrule
  \texttt{ZOIF\_dia} & 25.29 & 35.15 & 79.70 & 46.71 $\pm$ 0.50 \\
  \texttt{ZOIF\_hvp} & 25.45 & 36.16 & 82.86 & 48.16 $\pm$ 0.77 \\
  \texttt{ZOIF\_ide} & \best{26.22} & 35.57 & \best{83.58} & \best{48.46} $\pm$ 1.00 \\
  \texttt{ZOIF\_opt} & 25.93 & 35.24 & 83.33 & 48.17 $\pm$ 0.61 \\
  \texttt{FOIF\_dia} & 25.53 & 36.31 & 82.84 & 48.23 $\pm$ 0.35 \\
  \texttt{FOIF\_hvp} & 25.36 & 36.84 & 82.50 & 48.24 $\pm$ 0.98 \\
  \texttt{FOIF\_ide} & 25.85 & 35.27 & 81.84 & 47.65 $\pm$ 0.96 \\
  \texttt{FOIF\_opt} & 25.11 & 35.37 & 77.58 & 46.02 $\pm$ 0.79 \\
  \texttt{EMB} & 25.66 & \best{37.14} & 80.91 & 47.90 $\pm$ 0.49 \\
  \texttt{Top\_PPL} & 25.30 & 32.68 & 76.15 & 44.71 $\pm$ 0.32 \\
  \texttt{Mid\_PPL} & 25.42 & 33.70 & 72.95 & 44.02 $\pm$ 0.18 \\
  \texttt{Random} & 25.40 & 36.63 & 78.86 & 46.96 $\pm$ 0.36 \\
  \texttt{Orig} & 25.24 & 32.26 & 42.87 & 33.46 \\
\bottomrule
\end{tabular}
\caption{Held-out test performance on \texttt{Llama-2-7B} (Strategy 1). The macro-average entry is reported as mean $\pm$ standard deviation over three independent runs.}
\label{tab:llama-2-7B_results}
\end{table}

\begin{table}[htbp]
\centering
\begin{tabular}{lcccc}
\toprule
\textbf{Method} & \textbf{\shortstack{MMLU\\Acc.}} & \textbf{\shortstack{BBH\\Acc.}} & \textbf{\shortstack{SQuAD\\F1}} & \textbf{\shortstack{Macro\\Avg.}} \\
\midrule
  \texttt{ZOIF\_dia} & \best{25.77} & 35.85 & 80.32 & 47.31 $\pm$ 0.32 \\
  \texttt{ZOIF\_hvp} & 25.51 & 34.89 & 83.18 & 47.86 $\pm$ 0.36 \\
  \texttt{ZOIF\_ide} & 25.06 & 35.02 & 69.44 & 43.17 $\pm$ 0.81 \\
  \texttt{ZOIF\_opt} & 25.60 & 35.16 & 82.49 & 47.75 $\pm$ 1.32 \\
  \texttt{FOIF\_dia} & 25.73 & 33.81 & 83.81 & 47.78 $\pm$ 0.46 \\
  \texttt{FOIF\_hvp} & 25.39 & 34.37 & 82.49 & 47.42 $\pm$ 0.11 \\
  \texttt{FOIF\_ide} & 25.16 & 34.98 & \best{84.54} & 48.22 $\pm$ 0.96 \\
  \texttt{FOIF\_opt} & 25.42 & 34.39 & 75.78 & 45.20 $\pm$ 0.89 \\
  \texttt{EMB} & 25.26 & 36.06 & 84.09 & \best{48.47} $\pm$ 0.20 \\
  \texttt{Top\_PPL} & 25.23 & 32.91 & 75.88 & 44.67 $\pm$ 0.32 \\
  \texttt{Mid\_PPL} & 25.39 & 33.74 & 72.71 & 43.95 $\pm$ 0.18 \\
  \texttt{Random} & 25.41 & \best{36.39} & 79.00 & 46.93 $\pm$ 0.36 \\
  \texttt{Orig} & 25.24 & 32.26 & 42.87 & 33.46 \\
\bottomrule
\end{tabular}
\caption{Held-out test performance on \texttt{Llama-2-7B} (Strategy 2). The macro-average entry is reported as mean $\pm$ standard deviation over three independent runs.}
\label{tab:llama-2-7B_alt}
\end{table}

\subsubsection{Results on \texttt{Qwen}}\label{qw2}

Fine-tuning and evaluation results for \texttt{Qwen3-0.6B} (Strategy~2) and \texttt{Qwen3-1.7B} (Strategies~1 and~2) are reported in Tables~\ref{tab:Qwen3-0.6B_alt}, \ref{tab:Qwen3-1.7B_results}, and~\ref{tab:Qwen3-1.7B_alt}.
For \texttt{Qwen3-0.6B} (Strategy~2), \texttt{ZOIF\_opt} achieves the highest average score (51.72), outperforming \texttt{LESS} (50.16) and yielding a marginally higher average than \texttt{TRAK} (51.67). On \texttt{Qwen3-1.7B}, \texttt{ZOIF} consistently matches \texttt{FOIF} in overall performance, demonstrating robust results across model scales and strategies.

\begin{table}[htbp]
\centering
\begin{tabular}{lcccc}
\toprule
\textbf{Method} & \textbf{\shortstack{MMLU\\Acc.}} & \textbf{\shortstack{BBH\\Acc.}} & \textbf{\shortstack{SQuAD\\F1}} & \textbf{\shortstack{Macro\\Avg.}} \\
\midrule
  \texttt{ZOIF\_dia} & 43.32 & 35.06 & 75.55 & 51.31 $\pm$ 0.45 \\
  \texttt{ZOIF\_hvp} & 44.33 & 33.13 & 76.05 & 51.17 $\pm$ 0.81 \\
  \texttt{ZOIF\_ide} & 44.21 & 34.64 & 72.78 & 50.54 $\pm$ 0.94 \\
  \texttt{ZOIF\_opt} & 45.12 & 34.53 & 75.52 & \best{51.72} $\pm$ 0.65 \\
  \texttt{FOIF\_dia} & 43.04 & 34.26 & 73.91 & 50.41 $\pm$ 0.60 \\
  \texttt{FOIF\_hvp} & 44.35 & 33.96 & 70.88 & 49.73 $\pm$ 1.15 \\
  \texttt{FOIF\_ide} & 44.47 & 34.85 & 73.23 & 50.85 $\pm$ 0.40 \\
  \texttt{FOIF\_opt} & 42.89 & 33.70 & 75.53 & 50.71 $\pm$ 0.93 \\
  \texttt{LESS} & 43.15 & 33.22 & 74.11 & 50.16 $\pm$ 0.31 \\
  \texttt{TRAK} & 44.98 & 33.42 & \best{76.60} & 51.67 $\pm$ 0.49 \\
  \texttt{EMB} & 45.65 & \best{35.21} & 73.59 & 51.48 $\pm$ 0.16 \\
  \texttt{Top\_PPL} & 45.75 & 34.00 & 71.52 & 50.43 $\pm$ 0.27 \\
  \texttt{Mid\_PPL} & 45.82 & 33.83 & 72.31 & 50.66 $\pm$ 0.24 \\
  \texttt{Random} & 45.27 & 34.33 & 69.51 & 49.70 $\pm$ 0.37 \\
  \texttt{Orig} & \best{47.09} & 30.54 & 48.45 & 42.03 \\
\bottomrule
\end{tabular}
\caption{Held-out test performance on \texttt{Qwen3-0.6B} (Strategy 2). The macro-average entry is reported as mean $\pm$ standard deviation over three independent runs.}
\label{tab:Qwen3-0.6B_alt}
\end{table}

\begin{table}[htbp]
\centering
\begin{tabular}{lcccc}
\toprule
\textbf{Method} & \textbf{\shortstack{MMLU\\Acc.}} & \textbf{\shortstack{BBH\\Acc.}} & \textbf{\shortstack{SQuAD\\F1}} & \textbf{\shortstack{Macro\\Avg.}} \\
\midrule
  \texttt{ZOIF\_dia} & 57.55 & 41.46 & 82.49 & 60.50 $\pm$ 0.37 \\
  \texttt{ZOIF\_hvp} & 57.38 & 41.85 & 80.16 & 59.80 $\pm$ 0.28 \\
  \texttt{ZOIF\_ide} & 57.68 & \best{42.75} & 81.08 & 60.50 $\pm$ 0.61 \\
  \texttt{ZOIF\_opt} & 57.36 & 41.78 & 83.20 & 60.78 $\pm$ 0.16 \\
  \texttt{FOIF\_dia} & 58.25 & 42.37 & \best{83.99} & \best{61.53} $\pm$ 0.59 \\
  \texttt{FOIF\_hvp} & 58.01 & 42.25 & 82.58 & 60.95 $\pm$ 1.06 \\
  \texttt{FOIF\_ide} & 58.54 & 41.66 & 81.43 & 60.54 $\pm$ 1.04 \\
  \texttt{FOIF\_opt} & 57.81 & 40.33 & 82.30 & 60.15 $\pm$ 0.65 \\
  \texttt{EMB} & 58.34 & 42.37 & 83.04 & 61.25 $\pm$ 0.33 \\
  \texttt{Top\_PPL} & 57.59 & 41.94 & 82.70 & 60.74 $\pm$ 0.65 \\
  \texttt{Mid\_PPL} & 58.59 & 41.11 & 81.57 & 60.42 $\pm$ 0.56 \\
  \texttt{Random} & 58.85 & 42.42 & 81.93 & 61.07 $\pm$ 0.44 \\
  \texttt{Orig} & \best{59.88} & 40.50 & 40.66 & 47.01 \\
\bottomrule

\end{tabular}
\caption{Held-out test performance on \texttt{Qwen3-1.7B} (Strategy 1). The macro-average entry is reported as mean $\pm$ standard deviation over three independent runs.}
\label{tab:Qwen3-1.7B_results}
\end{table}

\begin{table}[htbp]
\centering
\begin{tabular}{lcccc}
\toprule
\textbf{Method} & \textbf{\shortstack{MMLU\\Acc.}} & \textbf{\shortstack{BBH\\Acc.}} & \textbf{\shortstack{SQuAD\\F1}} & \textbf{\shortstack{Macro\\Avg.}} \\
\midrule
  \texttt{ZOIF\_dia} & 57.81 & 41.51 & 82.25 & 60.52 $\pm$ 0.40 \\
  \texttt{ZOIF\_hvp} & 57.60 & 41.08 & 83.60 & 60.76 $\pm$ 0.45 \\
  \texttt{ZOIF\_ide} & 57.15 & 41.04 & 81.81 & 60.00 $\pm$ 1.17 \\
  \texttt{ZOIF\_opt} & 57.84 & 41.34 & 82.52 & 60.57 $\pm$ 0.47 \\
  \texttt{FOIF\_dia} & 58.52 & 42.03 & \best{83.88} & \best{61.47} $\pm$ 0.80 \\
  \texttt{FOIF\_hvp} & 56.81 & 40.36 & 83.39 & 60.19 $\pm$ 0.61 \\
  \texttt{FOIF\_ide} & 57.86 & \best{42.11} & 78.82 & 59.60 $\pm$ 2.10 \\
  \texttt{FOIF\_opt} & 57.43 & 40.86 & 81.65 & 59.98 $\pm$ 0.43 \\
  \texttt{EMB} & 58.55 & 40.77 & 83.25 & 60.86 $\pm$ 0.14 \\
  \texttt{Top\_PPL} & 57.59 & 41.94 & 82.70 & 60.74 $\pm$ 0.65 \\
  \texttt{Mid\_PPL} & 58.59 & 41.11 & 81.57 & 60.42 $\pm$ 0.56 \\
  \texttt{Random} & 58.72 & 41.72 & 81.22 & 60.56 $\pm$ 0.44 \\
  \texttt{Orig} & \best{59.88} & 40.50 & 40.66 & 47.01 \\
\bottomrule

\end{tabular}
\caption{Held-out test performance on \texttt{Qwen3-1.7B} (Strategy 2). The macro-average entry is reported as mean $\pm$ standard deviation over three independent runs.}
\label{tab:Qwen3-1.7B_alt}
\end{table}

\subsection{Hyperparameter Sensitivity Analysis}\label{sensitiva}
\subsubsection{Sensitivity of \texorpdfstring{$K$ and $\alpha$}{K and alpha}}
The results are presented in Table~\ref{tab:hyper_sensitivity}.
In this sensitivity setting, the best observed performance occurs at $\alpha = 0.01$ for all four variants. A larger $\alpha$ can amplify truncation error due to the first-order approximation, whereas an excessively small $\alpha$ can make gradient estimation more vulnerable to floating-point noise and numerical instability. This result is specific to the evaluated \texttt{Llama-3.2-1B} SQuAD setting and does not imply that $\alpha = 0.01$ is universally optimal.

Regarding the number of perturbations $K$, the best observed performance generally occurs within the tested range of $10$ to $30$, although the trend varies across \texttt{ZOIF} variants. A larger direction budget reduces Monte Carlo error for the ZO estimator, while finite-radius bias and the structural Step~2 residual remain. The non-monotonic results therefore indicate diminishing and variant-dependent returns over the tested range. Theorem~\ref{theoremboundLowrH} further gives an effective-dimension interpretation when both the computation and perturbations are explicitly restricted to a low-rank subspace.

\begin{table}[tb]
\centering
\begin{tabular}{lccc}
\toprule
\multicolumn{4}{c}{$\alpha$ Sensitivity (fixed $K=30$)} \\
\textbf{\texttt{ZOIF} Variant} & $\alpha=0.001$ & $\alpha=0.01$ & $\alpha=0.1$ \\
\midrule
\texttt{ZOIF\_dia} & 62.40 & \textbf{69.36} & 65.19 \\
\texttt{ZOIF\_hvp} & 59.28 & \textbf{69.43} & 64.08 \\
\texttt{ZOIF\_ide} & 61.95 & \textbf{67.06} & 55.95 \\
\texttt{ZOIF\_opt} & 65.22 & \textbf{70.58} & 61.02 \\
\midrule
\multicolumn{4}{c}{$K$ Sensitivity (fixed $\alpha=0.001$)} \\
\textbf{\texttt{ZOIF} Variant} & $K=5$ & $K=10$ & $K=30$ \\
\midrule
\texttt{ZOIF\_dia} & 60.22 & 61.63 & \textbf{62.40} \\
\texttt{ZOIF\_hvp} & 44.41 & \textbf{65.98} & 59.28 \\
\texttt{ZOIF\_ide} & 52.17 & 56.57 & \textbf{61.95} \\
\texttt{ZOIF\_opt} & 61.19 & \textbf{67.31} & 65.22 \\
\bottomrule
\end{tabular}
\caption{Performance sensitivity to perturbation radius $\alpha$ (fixed $K=30$) and random directions $K$ (fixed $\alpha=0.001$) on Llama-3.2-1B (SQuAD).}
\label{tab:hyper_sensitivity}
\end{table}

\subsubsection{Selection Budget Sensitivity}

We further investigate the impact of data selection budget on model performance by conducting a sensitivity analysis on \texttt{Llama-3.2-1B} (SQuAD, Strategy~2) across selection ratios of $2.5\%$, $5.0\%$, $10.0\%$, and $15.0\%$. As summarized in Table~\ref{tab:budget-sensitivity}, the results reveal a non-monotonic trend: performance generally improves as the budget increases from $2.5\%$ to $10.0\%$, but declines at the $15.0\%$ ratio ($\sim$30k samples) for most \texttt{ZOIF} variants and the Random baseline. This observation suggests that data selection is not governed by a simple ``more-is-better'' principle; beyond an optimal threshold, expanding the selection budget may introduce redundancy or distributional shift, thereby diluting the quality of the curated subset and impairing fine-tuning efficacy.

\begin{table}[tb]
\centering
\begin{tabular}{lcccc}
\toprule
\textbf{Method} & \textbf{2.5\%} & \textbf{5.0\%} & \textbf{10.0\%} & \textbf{15.0\%} \\
\midrule
\texttt{ZOIF\_dia}       & 60.66          & 62.40          & \textbf{73.18}  & \textbf{69.69}  \\
\texttt{ZOIF\_hvp}       & 61.06          & 59.28          & 58.61           & 60.66           \\
\texttt{ZOIF\_ide}       & \textbf{64.80} & 61.95          & 64.96           & 62.91           \\
\texttt{ZOIF\_opt}       & 62.50          & \textbf{65.22} & 69.08           & 65.39           \\
\midrule
\texttt{Random}          & 58.80          & 62.41          & 65.23           & 64.33           \\
\bottomrule
\end{tabular}
\caption{Sensitivity analysis of selection budget on \texttt{Llama-3.2-1B} (SQuAD, Strategy~2).}
\label{tab:budget-sensitivity}
\end{table}

Across the majority of selection ratios, \texttt{ZOIF\_dia} and \texttt{ZOIF\_opt} consistently outperform random sampling, with peak gains observed at moderate budgets ($5.0\%$--$10.0\%$). Notably, \texttt{ZOIF\_ide} achieves the highest score at the most constrained setting ($2.5\%$), highlighting its robustness under extreme resource limitations. Collectively, these results corroborate \texttt{ZOIF}'s capacity to reliably identify high-influence samples and demonstrate its superior gain efficiency: under tight annotation or computational budgets, \texttt{ZOIF} delivers stronger performance per selected sample compared to naive baselines. The observed performance decay at $15.0\%$ further underscores the importance of adaptive budget calibration in practical deployment scenarios.

\subsection{Stability of Ranking Correlation (\texttt{ZOIF} vs. \texttt{FOIF})}\label{stabZOFO}

To assess the ranking stability of our approach, we conduct a comparative analysis between \texttt{ZOIF} and \texttt{FOIF} (used as the first-order reference) on Llama-3.2-1B using the SQuAD dataset. The evaluation employs two complementary metrics: (1) \textit{Top-$K$ Overlap}, which quantifies the agreement on the most influential samples, and (2) \textit{Spearman Rank Correlation at Intersection} ($\rho_{\cap}$), which measures the monotonic relationship within the co-selected Top-$K$ subset ($K=10{,}000$). The quantitative results are presented in Table~\ref{tab:correlation-analysis}.

\begin{table}[tb]
\centering
\begin{tabular}{lcc}
\hline
\textbf{Method} & \textbf{Top-$K$ Overlap} & \textbf{Spearman $\rho_{\cap}$} \\ \hline
\texttt{ZOIF\_dia} & 16.44\% & 0.3277 \\
\texttt{ZOIF\_hvp} & 19.74\% & 0.2725 \\
\texttt{ZOIF\_ide} & 27.41\% & 0.2871 \\
\texttt{ZOIF\_opt} & 19.72\% & 0.4096 \\ \hline
\end{tabular}
\caption{Consistency analysis: \texttt{ZOIF} vs. \texttt{FOIF} ($K=10{,}000$).}
\label{tab:correlation-analysis}
\end{table}

The results show measurable alignment between \texttt{ZOIF} and the reference \texttt{FOIF} rankings from both macroscopic and microscopic perspectives. Macroscopically, the Top-$K$ overlap spans $16.44\%$ to $27.41\%$, exceeding the expected random baseline of $\sim 5\%$ and indicating that \texttt{ZOIF} recovers part of the high-value sample regions identified by \texttt{FOIF}. Microscopically, the Spearman correlation peaks at $0.4096$ for the \texttt{ZOIF\_opt} variant, indicating moderate consistency in relative rankings despite the inherent noise of zeroth-order estimation. Together with the downstream results, this partial ranking agreement helps explain how \texttt{ZOIF} achieves performance comparable to \texttt{FOIF} while substantially reducing memory overhead.

\end{document}